\documentclass[twoside,natbib]{article} \usepackage[accepted]{aistats2017}

\usepackage{natbib}
\renewcommand{\cite}{\citep}

\usepackage{mathtools}

\usepackage{hyperref,color}
\definecolor{darkgreen}{rgb}{0,.4,.2}
\definecolor{darkblue}{rgb}{.1,.2,.6}
\definecolor{brightblue}{rgb}{0,0.6,0.8}
\hypersetup{
  colorlinks=true,
  linkcolor=darkblue,
  citecolor=darkblue,
  filecolor=darkgreen,
  urlcolor=darkgreen
}

\usepackage{amssymb}
\usepackage{amsmath}
\usepackage{amsthm}

\usepackage{algorithm}
\usepackage{algpseudocode}

\usepackage{subcaption}

\newtheorem{theorem}{Theorem}[section]

\newtheorem{remark}[theorem]{Remark}

\newtheorem{definition}[theorem]{Definition}
\newtheorem{lemma}[theorem]{Lemma}

\newtheorem*{rep@theorem}{\rep@title}
\newcommand{\newreptheorem}[2]{%
\newenvironment{rep#1}[1]{%
 \def\rep@title{#2 \ref{##1}}%
 \begin{rep@theorem}}%
 {\end{rep@theorem}}}
\newreptheorem{lemma}{Lemma'}
\newreptheorem{definition}{Definition'}
\newreptheorem{proposition}{Proposition'}
\newreptheorem{theorem}{Theorem'}

\algnewcommand\algorithmicswitch{\textbf{switch}}
\algnewcommand\algorithmiccase{\textbf{case}}
\algnewcommand\algorithmicassert{\texttt{assert}}
\algnewcommand\Assert[1]{\State \algorithmicassert(#1)}%
\algdef{SE}[SWITCH]{Switch}{EndSwitch}[1]{\algorithmicswitch\ #1\ \algorithmicdo}{\algorithmicend\ \algorithmicswitch}%
\algdef{SE}[CASE]{Case}{EndCase}[1]{\algorithmiccase\ #1}{\algorithmicend\ \algorithmiccase}%
\algtext*{EndSwitch}%
\algtext*{EndCase}%

\newcommand{\R}{\mathbb{R}}
\newcommand{\E}{\mathbb{E}}

\newcommand{\av}{\boldsymbol{a}}

\newcommand{\ev}{\boldsymbol{e}}

\newcommand{\pv}{\boldsymbol{p}}

\newcommand{\uv}{\boldsymbol{u}}
\newcommand{\vv}{\boldsymbol{v}}
\newcommand{\wv}{\boldsymbol{w}}
\newcommand{\xv}{\boldsymbol{x}}
\newcommand{\yv}{\boldsymbol{y}}

\newcommand{\Fv}{\boldsymbol{F}}
\newcommand{\Gv}{\boldsymbol{G}}

\newcommand{\OA}{\mathcal{O}_{\hspace{-1pt}A}}
\newcommand{\OB}{\mathcal{O}_{\hspace{-1pt}B}}

\newcommand{\0}{\boldsymbol{0}}

\newcommand{\alphav}{\boldsymbol{\alpha}}
\newcommand{\kappav}{\boldsymbol{\kappa}}

\DeclarePairedDelimiter\abs{\lvert}{\rvert}

\renewcommand{\epsilon}{\ensuremath\varepsilon}

\renewcommand{\phi}{\ensuremath{\varphi}}

\begin{document}

\twocolumn[

\aistatstitle{Faster Coordinate Descent via Adaptive Importance Sampling}

\aistatsauthor{ Dmytro Perekrestenko \And Volkan Cevher \And Martin Jaggi }

\aistatsaddress{ ETH Zurich \And EPFL \And EPFL } ]

\begin{abstract}
  \vspace{-3mm}
Coordinate descent methods employ random partial updates of decision variables in order to solve huge-scale convex optimization problems. 
In this work, we introduce new \textit{adaptive} rules for the random selection of their updates. By adaptive, 
we mean that our selection rules are based on the dual residual or the primal-dual gap estimates and can change at each iteration. We theoretically characterize 
the performance of our selection rules and demonstrate improvements over the state-of-the-art, and extend our theory and algorithms to general convex objectives. Numerical evidence with hinge-loss support vector 
machines and Lasso confirm that the practice follows the theory. 
\end{abstract}

 \vspace{-3mm}
\section{Introduction}
 \vspace{-3mm}
Coordinate descent methods rely on random partial updates of decision variables for scalability. 
Indeed, due to their space and computational efficiency as well as their ease of implementation, 
these methods are the state-of-the-art for a wide selection of standard machine learning and 
signal processing applications \citep{Fu:1998cd,Hsieh:2008bd,Wright:2015bn}.

Basic coordinate descent methods sample an active coordinate set for optimization 
uniformly at random, \textit{cf.}, Stochastic Dual Coordinate Ascent (SDCA) \cite{ShalevShwartz:2013wl} 
and other variants \cite{Friedman:2007ut,Friedman:2010wm,ShalevShwartz:2011vo}. 
However, recent results suggest that by employing an appropriately defined \emph{non-uniform} fixed sampling strategy, 
the convergence can be improved both in the theory as well as in practice  \cite{Zhao:2014vg, Necorara:2012wt,Nesterov:2012fa}.

In this work, we show that we can surpass the existing convergence rates 
by exploiting adaptive sampling strategies that change the sampling probability 
distribution during each iteration. For this purpose, we adopt the primal-dual 
framework of \citet{Dunner:2016vga}. In contast, however, we also 
handle convex optimization problems with general convex regularizers without
assuming strong convexity of the regularizer. 

In particular, we consider an adaptive coordinate-wise duality gap based sampling. 
Hence, our work can be viewed  as a natural continuation of the work of \citet{Csiba:2015ue}, 
where the authors introduce an adaptive version of SDCA for the  
smoothed hinge-loss support vector machine (SVM). However, our work generalizes %
the gap-based adaptive criterion of \cite{Osokin:2016tp} %
in a nontrivial way  
to a broader convex optimization template of the following form: 
\begin{equation}\label{eq: template}
   \min_{\alphav\in\R^n}\  f(A\alphav) + \sum_i g_i(\alpha_i), 
\vspace{-1mm}
\end{equation}
where $A$ is the data matrix, $f$ is a smooth convex function, and each $g_i$ is a general convex function. 

The template problem class in \eqref{eq: template} includes not only smoothed hinge-loss SVM, 
but also Lasso, Ridge Regression, (the dual formulation of the) original hinge-loss SVM, 
Logistic Regression, etc. As a result, our theoretical results for adaptive sampling can also recover 
the existing results for fixed non-uniform \cite{Zhao:2014vg} and uniform \cite{Dunner:2016vga} 
sampling as special cases.

\textbf{Contributions.} Our contributions are as follows: \vspace{-3mm}
\begin{itemize}
\item We introduce new adaptive and fixed non-uniform sampling schemes for random coordinate descent for problems for the template \eqref{eq: template}.
\vspace{-1mm}

\item To our knowledge, we provide the first convergence rate analysis of coordinate descent methods with adaptive sampling for problems with general convex regularizer (i.e., the class in \eqref{eq: template}). 
\vspace{-1mm}

\item We derive convergence guarantees with arbitrary sampling distributions for both strongly convex and the general convex cases, and identify new convergence improvements. %
\vspace{-1mm}

\item We support the theory with numerical evidence (i.e., Lasso and hinge-loss SVM) and illustrate significant performance improvements in practice. %
\vspace{-1mm}
\end{itemize}
\textit{Outline:} Section 2 provides basic theoretical preliminaries. Section 3 describes our theoretical results and introduces new sampling schemes. Section 4 discusses the application of our theory to machine learning, and compares the computational complexity of proposed sampling methods. Section 5 provides numerical evidence for the new methods. Section 6 discusses our contributions in the light of existing work. 

\section{Preliminaries}
\vspace{-1mm}
We recall some concepts from convex optimization used in the sequel. The \emph{convex conjugate} of a function $f: \R^n\rightarrow \R$ is defined as
$f^*(\vv) := \sup_{\uv\in\R^n} \vv^\top \uv - f(\uv)$.

\begin{definition} A function $f$: $\R^n \rightarrow \R \cup \{+\infty\}$ has $B$-bounded support if its domain is bounded by $B$:
\vspace{-1mm}
\[
f(\uv) < +\infty \rightarrow \|\uv\| \leq B.
\]
\end{definition}

\begin{lemma}[{Duality between Lipschitzness and L-Bounded Support, \cite{Dunner:2016vga}}] 
\label{duality_lip-bound}
Given a proper convex function $g$, it holds that $g$ has $L$-bounded support if and only if $g^*$ is $L$-Lipschitz.
\end{lemma}

\subsection{Our primal-dual setup}
\vspace{-1mm}
In this paper, we develop coordinate descent methods for the following primal-dual optimization pair:
\begin{align}
&\min_{\alphav\in \R^n} \Big [ \OA(\alphav):= f(A\alphav) + \sum_i g_i(\alpha_i) \Big ], ~\tag{A}\label{eq:A}\\
&\min_{\wv \in \R^d} \Big [ \OB(\wv):= f^*(\wv) + \sum_i g_i^*(-\av_i^\top\wv) \Big ], ~\tag{B}\label{eq:B}
\end{align}
where we have $A = [\av_1, \dots, \av_n]$. 

Our primal-dual template generalizes the standard primal-dual SDCA setup in \cite{ShalevShwartz:2013wl}. As a result, we can provide gap certificates for the quality of the numerical solutions while being applicable to broader set of problems. %
\vspace{-1mm}
\paragraph{Optimality conditions.}
The first-order optimality conditions for the problems \eqref{eq:A} and \eqref{eq:B} are given by
\begin{equation}
\label{opt_conditions}
\begin{aligned}
&\wv \in  \partial f(A \alphav), \hspace{7mm} -\av_i^\top \wv \in  \partial g_i(\alpha_i) &\forall i \in [n] \\
&A\alphav \in \partial f^*(\wv), \hspace{7mm}  \alpha_i \in   \partial g_i^*(-\av_i^\top \wv) &\forall i \in [n]
\end{aligned}
\end{equation}
For a proof, see \cite{Bauschke:2011ik}.

\paragraph{Duality gap.}
The duality gap is the difference between primal and dual solutions: 
\begin{equation}
\label{duality_gap}
G(\alphav,\wv) := \OA(\alphav) - (-\OB(\wv)),
\end{equation}
which provides a certificate on the approximation accuracy both the primal and dual objective values. 

While always non-negative, under strong duality the gap reaches zero only in an optimal pair $(\alphav^\star,\wv^\star)$. When $f$ is differentiable the optimality conditions \eqref{opt_conditions} write as $\wv^\star=  \wv(\alphav^\star) = \nabla  f(A \alphav^\star)$.
We will rely on the following relationship in our applications:
\[
\wv=\wv(\alphav) := \nabla f(A \alphav)\ .
\]%
Choosing this mapping allows for running existing algorithms based on $\alphav$ alone, without the algorithm needing to take care of $\wv$. Nevertheless, the mapping allows us to express the gap purely in terms of the original variable $\alphav$, at any time: We define $G(\alphav) :=G(\alphav,\wv(\alphav)) = \OA(\alphav) - (-\OB(\wv(\alphav)))$.

\paragraph{Coordinate-wise duality gaps.}
\label{sec:cw-gaps}
For our problem structure of partially separable problems \eqref{eq:A} and \eqref{eq:B}, it is not hard to show that the duality gap can be written as a sum of coordinate-wise gaps:
\begin{equation}
\label{eq:gap_decomposition}
G(\alphav) = \sum_i G_i(\alpha_i) := \sum_i \Big (  g_i^*(-\av_i^\top\wv) + g_i(\alpha_i) + \alpha_i \av_i^\top \wv \Big)
\end{equation}
The relation holds since our mapping $\wv = \nabla f(A\alphav)$ invokes the Fenchel-Young inequality for $f$ with an equality. Moreover, the Fenchel-Young inequality for $g_i$ implies that all $G_i(\alpha_i)$'s are non-negative.

\subsection{Dual residuals}
\vspace{-1mm}
We base our fixed non-uniform and adaptive schemes on the concept of ``dual residual,'' i.e., a measure of progress to the optimum of the dual variables $\alphav$. Here we assume that $\wv = \nabla f(A\alphav)$.
\begin{definition}[Dual Residual. A generalization of  \cite{Csiba:2015ue}]
\label{kappa} 
Consider the primal-dual setting \eqref{eq:A}-\eqref{eq:B}. Let each $g_i$ be $\mu_i$-strongly convex with convexity parameter $\mu_i \geq 0$ $\forall i \in [n]$. For the case $\mu_i = 0$ we require $g_i$ to have a bounded support. Then, given $\alphav$, the $i$-th dual residue on iteration $t$ is given by:  $$ \kappa_i^{(t)} := \min_{u \in \partial g_i^*(-\av_i^\top\wv^{(t)})} |u - \alpha_i^{(t)}|.$$
\end{definition}

\begin{remark}
Note that for $u$ to be well defined, i.e., the subgradient in (6) not to be empty, we need the domain of $g^*$ to be the whole space. For $\mu>0$ this is given by strong convexity of $g_i$, while for $\mu_i = 0$ this follows from the bounded support assumption on $g_i$.
\end{remark}

 \begin{definition}[Coherent probability vector, \cite{Csiba:2015ue}] We say that probability vector $\pv^{(t)} \in \R^n$ is coherent with the dual residue vector $\kappav^{(t)}$ if for all $i \in [n]$, we have 
 $ \kappa_i^{(t)} \neq 0 \hspace{0.1cm} \Rightarrow \hspace{0.1cm} p_i^{(t)} > 0.$
\end{definition}

 \begin{definition}[$t$-support set] 
 \label{sset}
 We call the set\vspace{-2mm}
 $$ I_t := \{ i \in [n] : \kappa_i^{(t)} \neq 0 \}  \subseteq [n]\vspace{-2mm}$$ 
 a $t$-support set.
\end{definition}

\begin{lemma} 
\label{kappa_bound} 
Suppose that for each $i$, $g_i^*$ is $L_i$-Lipschitz. 
Then, $\forall i:$ $|\kappa_i| \leq 2L_i$.
\vspace{-2mm}
\end{lemma}
\begin{proof}
By Lemma \ref{duality_lip-bound}, the $L_i$-Lipschitzness $g_i^*$ implies $L_i$-bounded support of $g_i(\alpha_i)$ and therefore $|\alpha_i| \leq L_i$. By writing Lipschitzness as bounded subgradient, %
$|u_i| \leq L_i$, and $|\kappa_i| = |\alpha_i - u_i| \leq |\alpha_i| + |u_i| \leq 2L_i$.
\end{proof}

\subsection{Coordinate Descent}
Algorithm~\ref{CD} describes the Coordinate Descent (CD) method in the primal-dual setting \eqref{eq:A} and \eqref{eq:B}. The  method has $3$ major steps: Coordinate selection, line-search along the chosen coordinate, and primal-dual parameter updates. While the standard CD methods chooses the coordinates at random with fixed distibutions, we  develop adaptive strategies in the sequel that change the sampling distribution per iteration. The other steps remain essentially the same. 

\begin{algorithm}
\caption{Coordinate Descent}
\label{CD}
\begin{algorithmic}[1]
\State Let $\alphav^{(0)}:=\0 \in \R^n, \wv^{(0)} := \wv(\alphav^{(0)})$
\For{ $t = 0,1,...,T$}
\State Sample $i \in [n]$ randomly according to $\pv^{(t)}$
\State Find $\Delta \alpha_i$ minimizing $\OA(\alphav^{(t)} + \ev_i\Delta \alpha_i)$
\State $\alphav^{(t+1)} := \alphav^{(t)} + \ev_i\Delta \alpha_i$
\State $\wv^{(t+1)} := \wv(\alphav^{(t+1)})$
\EndFor
\end{algorithmic}
\end{algorithm}
\vspace{-2mm}

\section{Adaptive Sampling-based CD}
Our goal is to find a $\varepsilon_B$-suboptimal parameter $\wv$ or $\varepsilon_A$-suboptimal parameter~$\alphav$, i.e., $\OA(\alphav) - \OA(\alphav^\star) \leq \varepsilon_A$ or $\OB(\wv) - \OB(\wv^\star) \leq \varepsilon_B$, for the following pair of dual optimization problems \eqref{eq:A} and \eqref{eq:B}.

\subsection{Key lemma}
This subsection introduces a lemma that characterizes the relationship between any sampling distribution for the coordinates, denoted as $\pv$, and the convergence rate of the CD method. For this purpose, we build upon the \cite[Lemma 3]{Csiba:2015ue}  to relax the strong-convexity restrictions on $g_i$'s. That is, we derive a convergence result for the general convex $g_i$ with coordinate-dependent strong convexity constants $\mu_i$. In contrast to \cite[Lemma 3]{Csiba:2015ue}, we can have $\mu_i=0$ when $g_i$ has bounded support. 

\begin{lemma}
\label{lemma1}
Let $f$ be $1/\beta$-smooth and each $g_i$ be $\mu_i$-strongly convex with convexity parameter $\mu_i \geq 0$ $\forall i \in [n]$. For the case $\mu_i = 0$, we require $g_i$ to have a bounded support. Then for any iteration $t$, any sampling distribution $\pv^{(t)}$ and any arbitrary $s_i \in [0,1]$ $\forall i \in [n]$, the iterates of the CD method satisfy 
\begin{equation}
\label{eq1}
\begin{aligned}
&\E[\OA(\alphav^{(t+1)}) \,|\, \alphav^{(t)}] \leq \OA(\alphav^{(t)}) - \sum_{i} s_i p_i^{(t)} G_i(\alphav^{(t)}) \\&-\sum_{i} p_i^{(t)} \Big ( \frac{\mu_i(s_i-s_i^2)}{2} - \frac{s_i^2 \|\av_i\|^2}{2 \beta}  \Big ) \abs{\kappa_i^{(t)}}^2,
\end{aligned}
\end{equation}
here $\kappa_i^{(t)}$ is $i$-th dual residual (see Def. \ref{kappa}).
\vspace{-1mm}
\end{lemma}
The proof is provided in Appendix \ref{app:proof}.

\begin{remark}
\label{remark1}
If in addition to the conditions of Lemma \ref{lemma1} we require $\pv^{(t)}$ to be coherent with $\kappav^{(t)}$, then for any $\theta \in [0,\min_{i \in I_t} p_i^{(t)}]$ it holds that 
\begin{equation}
\label{eq.lemma1}
\E[\OA(\alphav^{(t+1)})\,|\,\alphav^{(t)}] \leq \OA(\alphav^{(t)}) - \theta G(\alphav^{(t)}) + \frac{\theta^2 n^2}{2} F^{(t)},
\end{equation}
\begin{equation*}
 F^{(t)} := \frac{1}{n^2 \beta \theta}\sum_{i \in I_t} \Big ( \frac{\theta (\mu_i \beta + \|\av_i\|^2)}{p_i^{(t)}}  - \mu_i \beta \Big ) \abs{\kappa_i^{(t)}}^2.
\end{equation*}
\end{remark}
\begin{proof}
Since $s_i$  in Lemma \ref{lemma1} is an arbitrary number $\in [0,1]$, we take  $s_i= \frac{\theta}{p_i^{(t)}}$ for points with $i \in I_t$ and $s_i=0$ for all other points, here $\theta \in [0, \min_i p_i^{(t)}]$. Then,~\eqref{eq1} becomes the following, finalizing the proof:
$$\begin{aligned}
&\E[\OA(\alphav^{(t+1)})\,|\,\alphav^{(t)}] \leq \OA(\alphav^{(t)}) - \theta \sum_{i \in I_t} G_i(\alphav^{(t)}) \\&- \sum_{i \in I_t} \Big ( \frac{\mu_i \theta}{2} - \frac{\theta^2}{p_i^{(t)}}\frac{\mu_i \beta + \|\av_i\|^2}{2 \beta}  \Big ) \abs{\kappa_i^{(t)}}^2 \\ 
&= \OA(\alphav^{(t)}) -\theta G(\alphav^{(t)}) \\&- \frac{\theta}{2 \beta}\sum_{i \in I_t} \Big ( \mu_i \beta -  \frac{\theta (\mu_i \beta + \|\av_i\|^2)}{p_i^{(t)}} \Big ) \abs{\kappa_i^{(t)}}^2
\end{aligned}\vspace{-1em}$$
\end{proof}

\vspace{-6mm}
\subsection{Why is the generalization important?}
\label{sec:genconv}

Consider the key lemma with $\mu_i = 0$. Then, Remark~\ref{remark1} implies the following:\vspace{-1mm}
\begin{equation}
\label{bound}
\E[\OA(\alphav^{(t+1)})\,|\,\alphav^{(t)}] \leq \OA(\alphav^{(t)}) - \theta G(\alphav^{(t)}) + \frac{\theta^2 n^2}{2} F^{(t)},
\end{equation}
where 
\begin{equation} \label{F_genconv}
F^{(t)}  := \frac{1}{n^2 \beta}\sum_{i \in I_t} \Big ( \frac{  \abs{\kappa_i^{(t)}}^2 \|\av_i\|^2}{p_i^{(t)}} \Big ).
\end{equation}
Contrary to the strongly convex case, $F^{(t)}$ is  positive. Hence, the sampling distributions derived in \cite{Csiba:2015ue}  with strongly convex $g_i$ are not optimal. 

The following theorem generalizes \cite[Theorem 5]{Zhao:2014vg}, and \cite[Theorem 9]{Dunner:2016vga} to allow adaptive sampling: %
\begin{theorem}
\label{thm:main}
Assume $f$ is a $\frac{1}{\beta}$-smooth function. Then, if $g^*_i$ is $L_i$-Lipschitz for each $i$ and $\pv^{(t)}$ is coherent with $\kappav^{(t)}$, then the CD iterates satisfy
\vspace{-1mm}
\begin{equation}
\label{eq:bound1}
\E[\varepsilon_A^{(t)}] \leq \frac{2F^{\circ}n^2 + \frac{2\varepsilon_A^{(0)}}{p_{\min}}}{\frac{2}{p_{\min}} + t}.
\vspace{-1mm}
\end{equation}
Moreover, we obtain a \emph{duality gap} $G(\bar{\alphav}) \leq \varepsilon$ %
after an overall number of iterations $T$ whenever
\begin{equation}
\label{eq:theorem1}
T \geq \max\Bigg \{0, \frac{1}{p_{\min}} \log \Big (\frac{2 \varepsilon_A^{(0)}}{n^2 p_{\min} F^{\circ}} \Big) \Bigg\} + \frac{5F^{\circ}n^2}{\varepsilon} - \frac{1}{p_{\min}}.
\end{equation}

Moreover, when $t \geq T_0$ with 
\begin{equation}
\label{eq:theorem2}
T_0 := \max\Bigg \{0, \frac{1}{p_{\min}} \log \Big (\frac{2 \varepsilon_A^{(0)}}{n^2 p_{\min} F^{\circ}} \Big) \Bigg\}  +  \frac{4 F^{\circ} n^2}{\varepsilon} - \frac{2}{p_{\min}}
\end{equation}
we have the \emph{suboptimality} bound of  $\E[\OA(\alphav^{(t)}) - \OA(\alphav^{\star})]  \leq \varepsilon/2$.
Here $\varepsilon^{(0)}_A$ is the initial dual suboptimality and $F^{\circ}$ is an upper bound on $\E[F^{(t)}]$ taken over the random choice of the sampled coordinate at $1,\dots,T_0$ algorithm iterations.%
\end{theorem}
\vspace{-1mm}
The proof is provided in Appendix \ref{app:proof}.

\begin{remark}
We recover \cite[Theorem 9]{Dunner:2016vga} as a special case of Theorem \ref{thm:main} by setting $p_i^{(t)} = \frac{1}{n}$. We recover \cite[Theorem 5]{Zhao:2014vg} by setting $p_i^{(t)}=\frac{L_i}{\sum_j L_j}$.
\end{remark}

\subsubsection{Strategy I: Gap-wise sampling}
\label{gapwise_theory}
Based on the results above, we first develop sampling strategies for the CD method based on the decomposibly of the duality gap, i.e., sampling each coordinate according to its duality gap. 

\begin{definition}[Nonuniformity measure, \cite{Osokin:2016tp}]
\label{nonuniformity}
The nonuniformity measure $\chi(\xv)$ of a vector $\xv \in \R^n$, is defined as:
$$ \chi(\xv) := \sqrt{1+n^2 \text{Var}[\pv]},$$
where $\pv := \frac{\xv}{\| \xv\|_1}$ is the normalized probability vector. %
\end{definition}
\begin{lemma}
\label{lemmaNormRelation}
Let $\xv \in \R^n_+.$ Then, it holds that
\vspace{-1mm}
\[\|\xv\|_2 =  \frac{\chi(\xv)}{\sqrt{n}} \|\xv\|_1.
\vspace{-2mm}
\]
\end{lemma}
\begin{proof}
The proof follows from Def. \ref{nonuniformity} and \vspace{-1mm}
$$\textstyle\text{Var} [\pv] = \E[\pv^2] - \E[\pv]^2 = \frac{1}{n} \|\pv\|_2^2 - \frac{1}{n^2}.\vspace{-5mm}$$
\end{proof}
\begin{theorem}
\label{theorem_gapwise}
Let $f$ be a $\frac{1}{\beta}$-smooth function. Then, if $g^*_i$ is $L_i$-Lipschitz for each $i$ and $p_i^{(t)} := \frac{G_i(\alphav^{(t)})}{G(\alphav^{(t)})}$, then the iterations of the CD method satisfies
\begin{equation}
\E[\varepsilon_A^{(t)}] \leq \frac{2F^{\circ}_g n^2+2n\varepsilon_A^{(0)}}{t+2n},
\end{equation}
where $F^{\circ}_g$ is an upper bound on $\E\Big[ F^{(t)}_g\Big],$ where the expectation is taken over the random choice of the sampled coordinate at iterations $1,\dots, t$ of the algorithm. Here $\overrightarrow{\Gv}$ and $\overrightarrow{\Fv}$ are defined as follows:
$$\overrightarrow{\Gv}:= (G_i(\alphav^{(t)}))_{i=1}^n, \hspace{0.5cm} \overrightarrow{\Fv}:= ( \|\av_i \|^2 \abs{\kappa_i^{(t)}}^2)_{i=1}^n,$$
and $F^{(t)}_g$ is defined analogously to \eqref{F_genconv}:
\vspace{-1mm}
\begin{equation}
\label{F_gap}
 F^{(t)}_g:=\frac{\chi(\overrightarrow{\Fv})}{n \beta (\chi(\overrightarrow{\Gv}))^3} \sum_i \|\av_i\|^2 \abs{\kappa_i^{(t)}}^2. 
\vspace{-2mm}
\end{equation}
\end{theorem}
The proof is provided in Appendix \ref{app:proof}.

\paragraph{Gap-wise vs.\ Uniform Sampling:}
Here we compare the rates obtained by Theorem \ref{theorem_gapwise} for gap-wise sampling and Theorem \ref{thm:main} for uniform sampling. According to the Theorem \ref{thm:main}, the rate for any distribution can be written as follows
\vspace{-1mm}
\[
\E[\varepsilon_A^{(t)}] \leq \frac{2F^{\circ}n^2 + \frac{2\varepsilon_A^{(0)}}{p_{\min}}}{\frac{2}{p_{\min}} + t} = \frac{\frac{2}{\beta}\E\Big[ \sum_i \frac{  \abs{\kappa_i^{(t)}}^2 \|\av_i\|^2}{p_i^t}\Big] + \frac{2 \varepsilon_A^{(0)}}{p_{\min}} }{\frac{2}{p_{\min}} + t}.
\]
For the uniform distribution ($p_i = 1/n$), this yields
\begin{equation}
\E[\varepsilon_A^{(t)}] \leq \frac{\frac{2n}{\beta}\E\Big[ \sum_i \abs{\kappa_i^{(t)}}^2 \|\av_i\|^2\Big] + {2n \varepsilon_A^{(0)}} }{2n + t}.
\end{equation}
The rate of gap-wise sampling depends on non-uniformity measures $\chi(\overrightarrow{\Gv})$ and $\chi(\overrightarrow{\Fv})$:
$$ \E[\varepsilon_A^{(t)}] \leq \frac{\frac{2n}{\beta}\E\Big[\frac{\chi(\overrightarrow{\Fv})}{(\chi(\overrightarrow{\Gv}))^3}\sum_i \abs{\kappa_i^{(t)}}^2\|\av_i \|^2 \Big]+2n\varepsilon_A^{(0)}}{2n+t}.$$
In the best case for gap-wise sampling the variance in $( \abs{\kappa_i^{(t)}}^2 \|\av_i \|^2 )_{i=1}^n$ is $0$, $\chi(\overrightarrow{\Fv}) \approx 1$, and variance of gaps is maximal $\chi(\overrightarrow{\Gv}) \approx \sqrt{n}$. When this condition holds, the convergence rate becomes the following:
$$ \E[\varepsilon_A^{(t)}] \leq \frac{\frac{2}{\beta\sqrt{n}}\E\Big[\sum_i  \abs{\kappa_i^{(t)}}^2 \|\av_i \|^2 \Big]+2n\varepsilon_A^{(0)}}{2n+t}.$$
In the worst case scenario, when variance is maximal in $( \abs{\kappa_i^{(t)}}^2 \|\av_i \|^2 )_{i=1}^n$, $\chi(\overrightarrow{\Fv}) \approx \sqrt{n}$, the rate of gap-wise sampling is better than of uniform only when the gaps are non-uniform enough i.e., $\chi(\overrightarrow{\Gv}) \geq n^{\frac{1}{6}}$.

\subsection{Strategy II:  Adaptive  \& Uniform }
\label{mix_theory}
Instead of minimizing \eqref{eq:theorem1}, we here find an optimal sampling distribution as to minimize our bound:
\vspace{-1mm}
\begin{equation}
\label{iter}
T \geq \frac{5F^{\circ}n^2}{\varepsilon} + \frac{5\varepsilon_A^{(0)}}{\varepsilon p_{\min}}.
\end{equation}
The number of iterations $T$ is directly proportional to $F^{\circ}$ and $1/p_{\min}$.  Therefore, the optimal distribution $\pv$ should minimize $F^{\circ}$ and $1/p_{\min}$ at the same time. 

We denote the distribution minimizing $1/p_{\min}$ as \textit{supportSet uniform}, which is the following rule:
\vspace{-1mm}
\begin{equation}
\label{ssunifdist}
p_i^{(t)} :=
\begin{cases}
     \frac{1}{m_t},& \text{if } \kappa_i^{(t)} \neq 0 \\
     0,& \text{otherwise}.
\end{cases}
\end{equation}
Above, $m_t$ is a cardinality of the support set on iteration $t$. The distribution minimizing $F^{\circ}$, called \textit{adaptive}:
\vspace{-1mm}
\begin{equation} 
\label{adadist} 
p_i^{(t)} := \frac{\abs{\kappa_i^{(t)}} \|\av_i\|}{\sum_j |\kappa_j^{(t)}| \|\av_j\|}. 
\end{equation}
The mix of the two aforementioned distributions balances two terms and gives a good suboptimal $T$ in \eqref{iter}. We define mixed distribution as:
\begin{equation}
\label{mixdist}
p_i^{(t)} :=
\begin{cases}
     \frac{\sigma}{m_t} + (1-\sigma) \frac{\abs{\kappa_i^{(t)}} \|\av_i\|}{\sum_j |\kappa_j^{(t)}| \|\av_j\|},& \text{if } \kappa_i^{(t)} \neq 0 \\
     0,& \text{otherwise}
\end{cases}\vspace{-1mm}
\end{equation}
where $\sigma \in [0,1]$. This distribution gives us the following bounds on $F^{\circ}$ and $1/p_{\min}$:
\vspace{-1mm}
\[
F^{\circ}_{\text{mix}} \leq \frac{F^{\circ}_{\text{ada}}}{1-\sigma} \hspace{1cm} \frac{1}{p_{\min}} \leq \frac{m}{\sigma}.
\vspace{-1mm}
\]
and bound on the number of iterations:
\vspace{-1mm}
\begin{equation}
\label{mix_iter}
T \geq \frac{5F^{\circ}_{\text{ada}}n^2}{\varepsilon(1-\sigma)} + \frac{5\varepsilon_A^{(0)}m}{\varepsilon\sigma}.
\vspace{-1mm}
\end{equation}
Above $m:= \max_t m_t$. Since the process of finding $F^{\circ}_{\text{ada}}$ is rather problematic, a good $\sigma$ can be found by replacing $F^{\circ}_{\text{ada}}$ with its upper bound and minimizing \eqref{mixdist} w.r.t. $\sigma$. Another option is to use a ``safe'' choice of $\sigma = 0.5$, and provide a balance between two distributions. This strategy benefits from convergence  guarantees in case of unknown $F^{\circ}_{\text{ada}}$. In the applications section we use the latter option and call this sampling variant \textit{ada-uniform sampling}.

\subsection{Variations along the theme}
Based on the discussion above, we summarize our new variants of sampling schemes for Algorithm \ref{CD}:
\begin{itemize}
\item \textit{uniform} - sample uniformly at random.
\item \textit{supportSet uniform} - sample uniformly at random inside the support set, defined in \eqref{ssunifdist}. The  distribution is recomputed every iteration.
\item \textit{adaptive} - sample adaptively based on dual residual, defined in \eqref{adadist}. The  distribution is recomputed every iteration. 
\item \textit{ada-uniform} - sample based on a mixture between \textit{supportSet uniform} and \textit{adaptive}, defined in \eqref{mixdist}. The  distribution is recomputed every iteration. 
\item \textit{importance} - sample with a fixed non-uniform variant of \textit{adaptive} obtained by bounding $\kappa_i^{(t)}$ with $2L_i$ (Lemma \ref{kappa_bound}): $p_i := \frac{L_i \|\av_i\|}{\sum_j L_j \|\av_j\|}$. The distribution is computed only once. When the data is normalized, this sampling variant coincides with "importance sampling" of \cite{Zhao:2014vg}.\vspace{-1mm}
\item \textit{ada-gap} - sample randomly based on coordinate-wise duality gaps, defined in Section \ref{gapwise_theory}. The  distribution is recomputed every iteration. 
\item \textit{gap-per-epoch} - Use \textit{ada-gap} but with updates per-epoch. The gap-based  distribution is only recomputed at the beginning of each epoch and stays fixed during each epoch.
\end{itemize}
Full descriptions of the variants are  in Appendix \ref{app:algorithms}.
\section{Applications}
\label{applications}
\vspace{-1mm}

\subsection{Lasso and Sparse Logistic Regression}
The Lasso and L1-regularized Logistic Regression are quintessential problems with a general convex regularizer. Given a data matrix $A = [\av_1, \dots, \av_n] %
$ and a vector $\yv\in\R^d$, the Lasso is stated as:
\begin{equation}
\label{lasso_problem}
\min_{\alphav \in \R^n} %
\|A\alphav - \yv\|^2_2 + \lambda \|\alphav\|_1.
\end{equation} 
Both problems can easily be reformulated in our primal-dual setting \eqref{eq:A}-\eqref{eq:B}, by choosing $g_i(\alpha_i) := \lambda |\alpha_i|$. We have $f(A\alphav) :=  %
\|A\alphav - \yv\|^2_2$ for Lasso, and $f(A\alphav)$ is the logistic loss for classification respectively.
To our knowledge, there is no importance sampling or adaptive sampling techniques for  CD in this setting. %

\paragraph{Lipschitzing trick.}
In order to have duality gap convergence guarantees (Theorem \ref{thm:main}) we need $g_i^*$ to be Lipschitz continuous, which however is not the case for $g_i=|.|$ being the absolute value function. We modify the function $g_i$ without affecting the iterate sequence of CD using the ``Lipschitzing trick'' from \cite{Dunner:2016vga}, as follows.
 
According to Lemma \ref{duality_lip-bound}, a proper convex function~$g_i$ has bounded support if and only if $g_i^*$ is Lipschitz continuous. We modify $g_i(\alpha_i) = \lambda |\alpha_i|$ by restricting its support to the interval with radius $B:=\frac{1}{\lambda}(f(A\alphav^{(0)})+\lambda \|\alphav^{(0)}\|_1)$. Since Algorithm \ref{CD} is monotone, we can choose $B$ big enough to guarantee that $\alphav^{(t)}$ will stay inside the ball during optimization, i.e. that the algorithm's iterate sequence will not be affected by $B$. By modifying $g_i$ to bounded support of size $B$, we guarantee $g_i^*$ to be $B$-Lipschitz continuous.
\[
	\bar{g}_i(\alpha_i) := 
\begin{cases}
    \lambda |\alpha_i|, & \text{if } |\alpha_i| \leq B\\
    + \infty,              & \text{otherwise}
\end{cases}
\]
The conjugate of $\bar{g}_i$ is:\vspace{-1mm}
\[
\bar{g}_i^*(u_i) = \max_{\alpha_i : |\alpha_i| \leq B} u_i \alpha_i - \lambda |\alpha_i| = B \big [|u_i| - \lambda \big]_+ .
\]
\paragraph{Duality gap.}
Using the gap decomposition \eqref{eq:gap_decomposition} we obtain coordinate-wise duality gaps for modified Lasso and sparse logistic regression, which now depends on the chosen parameter~$B$:
\begin{align}
\label{eq:gap_decomposition_lasso}
G(\alphav) %
&= \sum_i \Big (  g_i^*(-\av_i^\top\wv) + g_i(\alpha_i) + \alpha_i \av_i^\top \wv \Big)\\ 
&= \sum_i \Big (  B\big [|\av_i^\top\wv| - \lambda \big]_+  +\lambda |\alpha_i| + \alpha_i \av_i^\top\wv \Big ).\notag
\end{align}

\subsection{Hinge-Loss SVM}
Our framework directly covers the original hinge-loss SVM formulation. 
The importance sampling technique \cite{Zhao:2014vg} are not applicable to the original hinge-loss, but relies on a smoothed version of the hinge-loss, changing the problem. %

When $\phi_i(.)$ is the hinge-loss, defined as $\phi_i(s) := [1 - sy_i ]_+$, our framework is directly applicable by mapping the SVM dual problem to our template \eqref{eq:A}, that is 
\begin{equation}
\label{svm_dual}
\min_{\alphav \in \R^n} \OA(\alphav): = \frac{1}{n} \sum_{i=1}^n \phi_i^*(-\alpha_i) + \frac{\lambda}{2} \Big \|\frac{1}{\lambda n}\sum_{i=1}^n\alpha_i\av_i \Big\|_2^2.
\end{equation}
The conjugate of the hinge-loss is $\phi_i^*(\alpha_i) = \alpha_i y_i $, with $\alpha_i y_i \in [0,1]$. In other words, $g_i^*(-\av_i^\top\wv) = \frac{1}{n}\phi_i(\av_i^\top\wv)$ in our notation, and $ f^*(\wv) = \frac{\lambda}{2} \|\wv\|_2^2$. %

\paragraph{Duality gap.} 
Section \ref{sec:cw-gaps} shows that the duality gap  decomposes into a sum of coordinate-wise gaps.

\subsection{Computational costs}
We discuss the computational costs of the proposed variants under the different sampling schemes. Table~\ref{table1} states the costs in detail, where \emph{nnz} is the number of non-zero entries in the data matrix $A$.
In the table, one \emph{epoch} means $n$ consecutive coordinate updates, where $n$ is the number features in the Lasso, and is the number of datapoints in the SVM.

\begin{table}[h]
\caption{A summary of computational costs}
\label{table1}
\centering
\begin{tabular}{|l| r|}
\hline
\textbf{Algorithm} & \textbf{Cost per Epoch}  \tabularnewline
\hline
uniform & $\mathcal O(\text{nnz})$  \tabularnewline
\hline
importance & $\mathcal O(\text{nnz} + n \log(n))$  \tabularnewline
\hline
gap-per-epoch & $\mathcal O(\text{nnz} + n \log(n))$  \tabularnewline
\hline
supportSet-uniform & $\mathcal O(n \cdot \text{nnz})$  \tabularnewline
\hline
adaptive & $\mathcal O(n \cdot \text{nnz})$   \tabularnewline
\hline
ada-uniform & $\mathcal O(n \cdot \text{nnz})$  \tabularnewline
\hline
ada-gap & $\mathcal O(n \cdot \text{nnz})$  \tabularnewline
\hline
\end{tabular}
\end{table}

\paragraph{Sampling and probability update}
In each iteration, we sample a coordinate %
from a non-uniform probability distribution. 
While the straightforward approach requires $\Theta(n)$ per sample, it is not hard to see that this can be improved to $\Theta(\log(n))$ when using a tree data structure to maintain the probability vector \cite{Nesterov:2013fn,ShalevShwartz:2016va}. The tree structure can be built in $\mathcal O(n \log(n))$.

\paragraph{Variable update and distribution generation}
Computing all dual residuals $\kappa_i$ or all coordinate-wise duality gaps $G_i$ is as expensive as an epoch of the classic CD method, i.e., we need to do $\Theta(\text{nnz})$ operations (one matrix-vector multiplication). 
In contrast, updating one coordinate $\alpha_i$ is cheap, being $\Theta(\text{nnz}/n)$.
 
\begin{figure*}
\centering
\rotatebox[origin=t]{90}{Lasso, \textit{\small rcv1*}}
\begin{subfigure}{.24\textwidth}
  \centering
\includegraphics[width=0.9\linewidth]{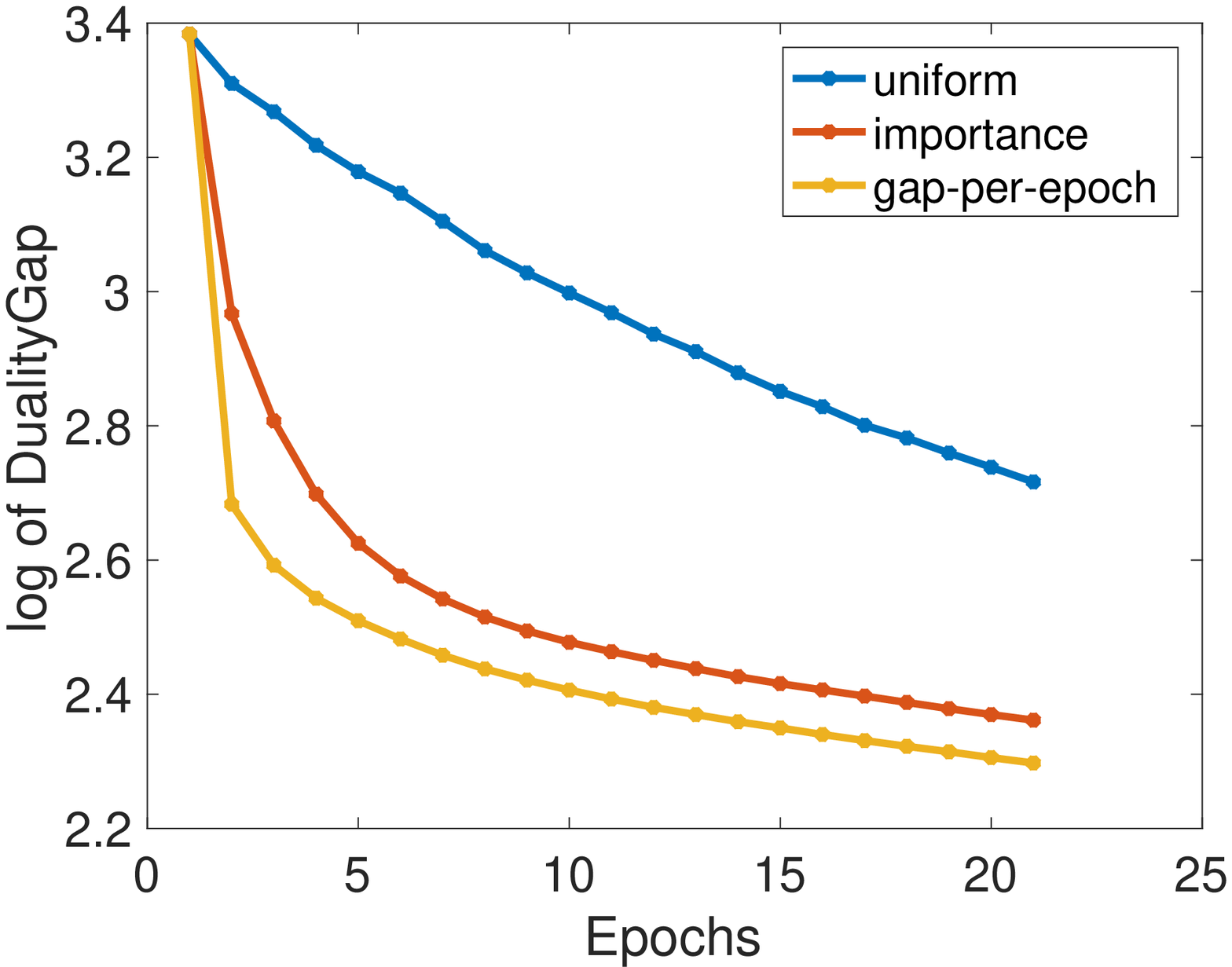}
  \caption{Gap, fixed distr.}
\end{subfigure}%
\begin{subfigure}{.24\textwidth}
  \centering
\includegraphics[width=0.9\linewidth]{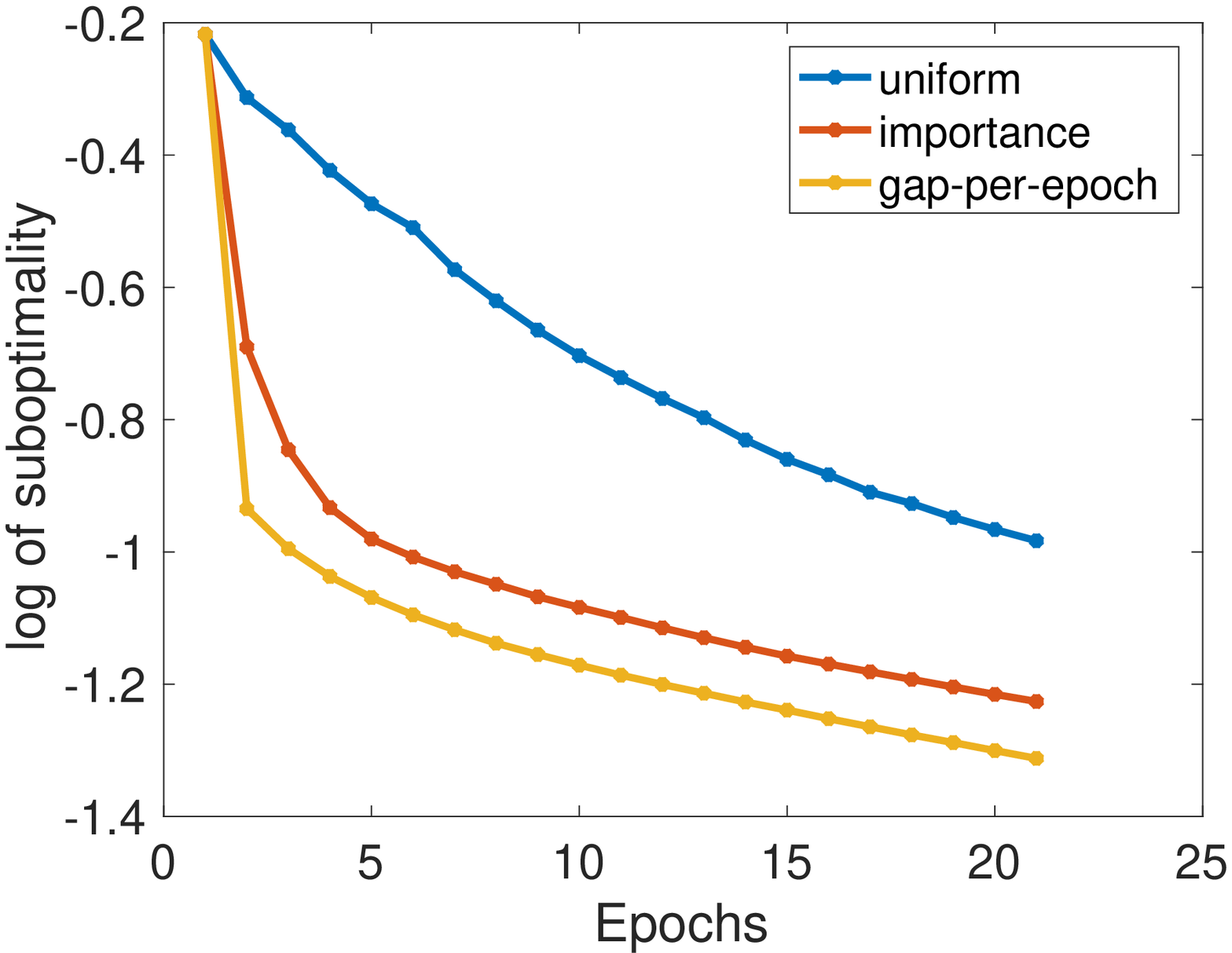}
  \caption{Subopt., fixed distr.}
\end{subfigure}%
\begin{subfigure}{.24\textwidth}
  \centering
\includegraphics[width=0.9\linewidth]{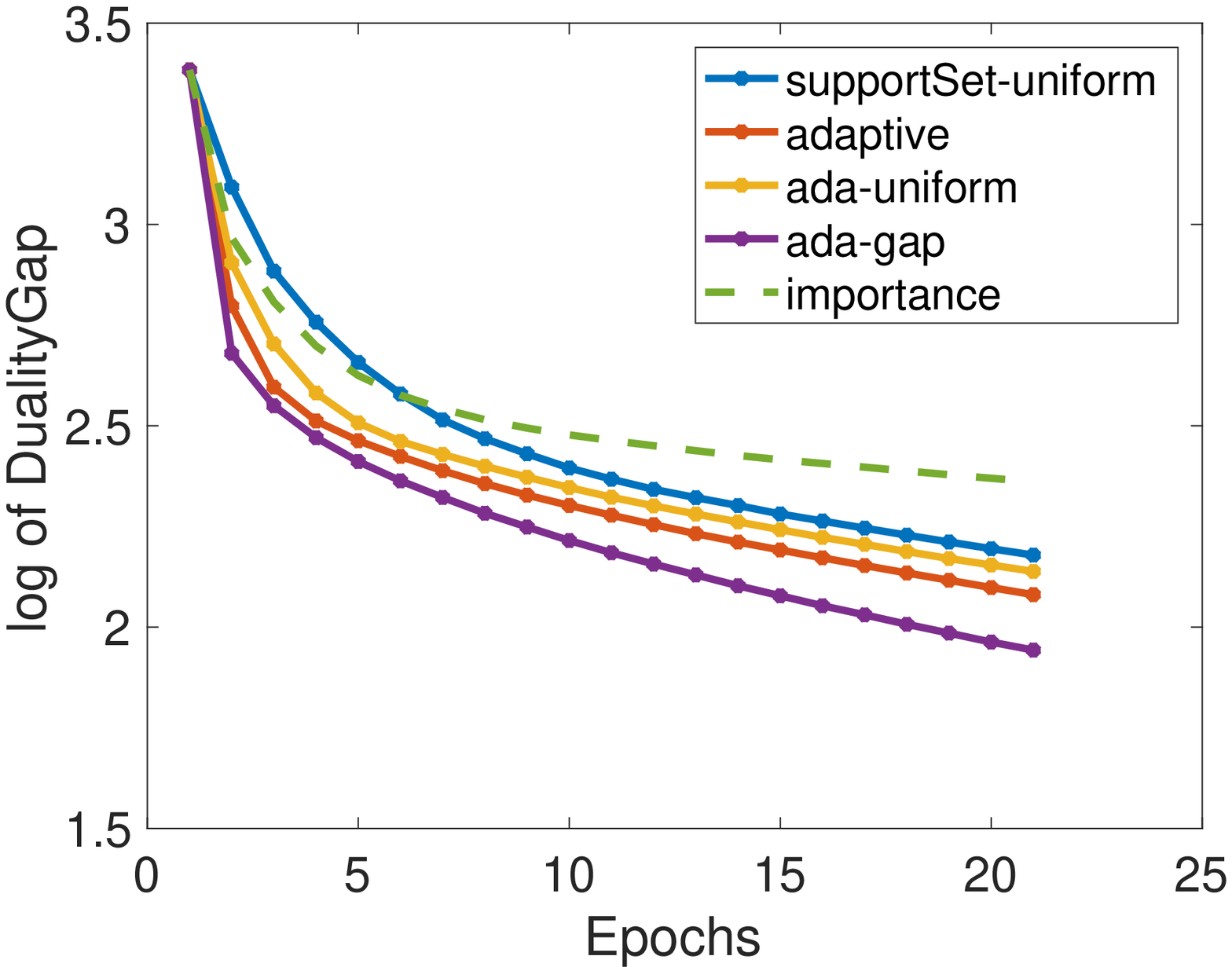}
  \caption{Gap, adaptive}
\end{subfigure}%
\begin{subfigure}{.24\textwidth}
  \centering
\includegraphics[width=0.9\linewidth]{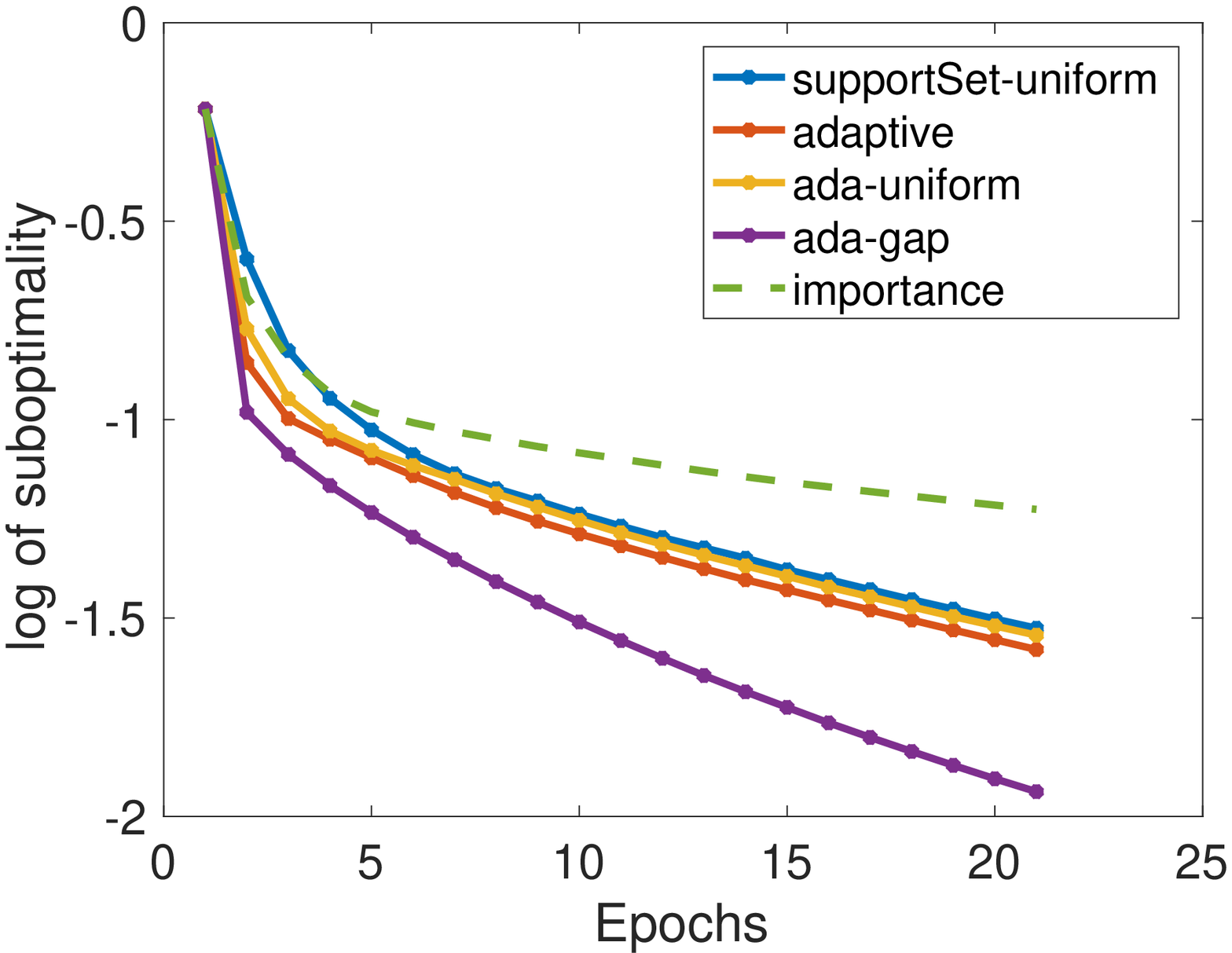}
  \caption{Suboptimality, adaptive}
\end{subfigure}%
\\
\rotatebox[origin=t]{90}{Lasso, \textit{\small mushrooms}}
\begin{subfigure}{.24\textwidth}
  \centering
\includegraphics[width=0.9\linewidth]{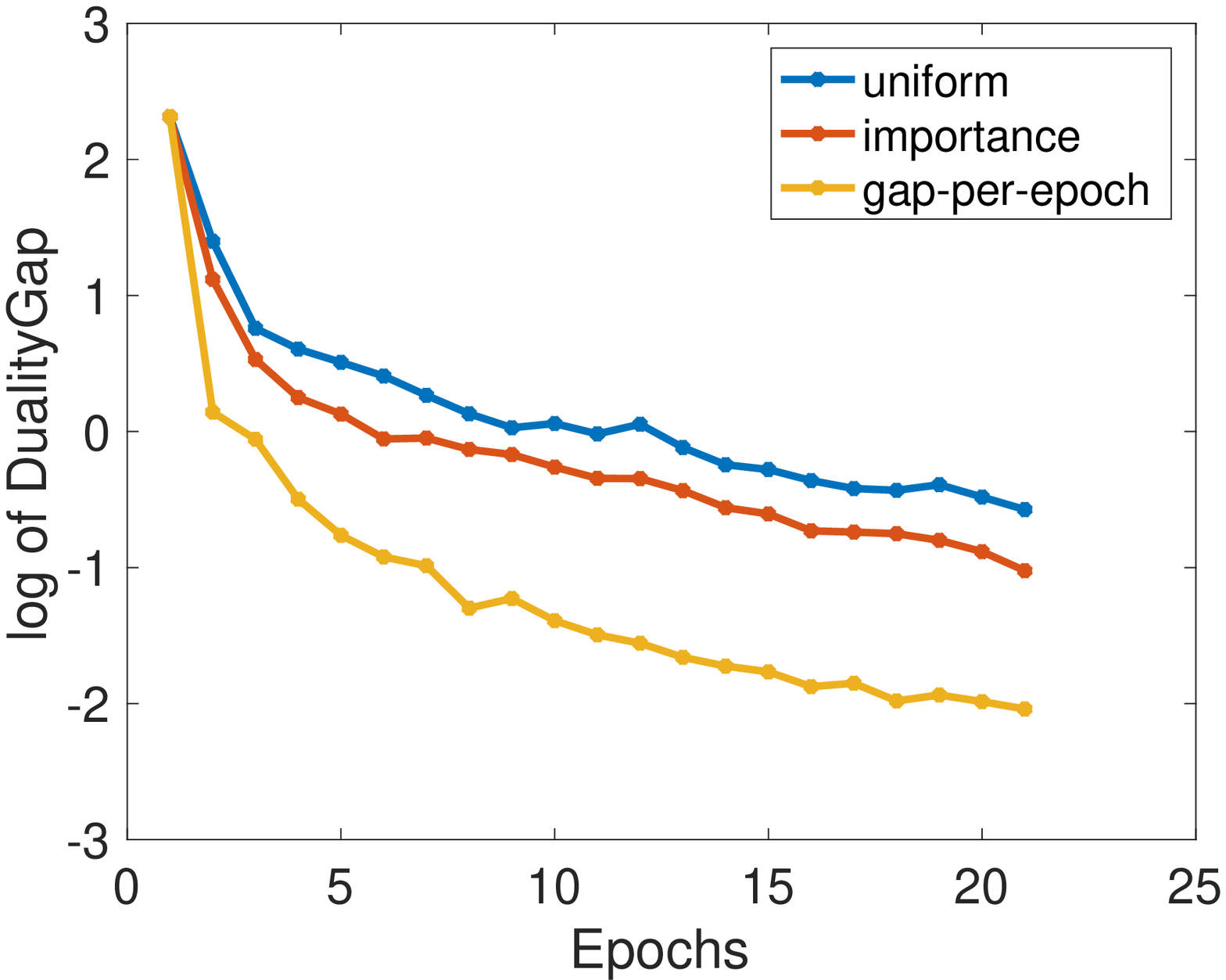}
  \caption{Gap, fixed distr.}
\end{subfigure}%
\begin{subfigure}{.24\textwidth}
  \centering
\includegraphics[width=0.9\linewidth]{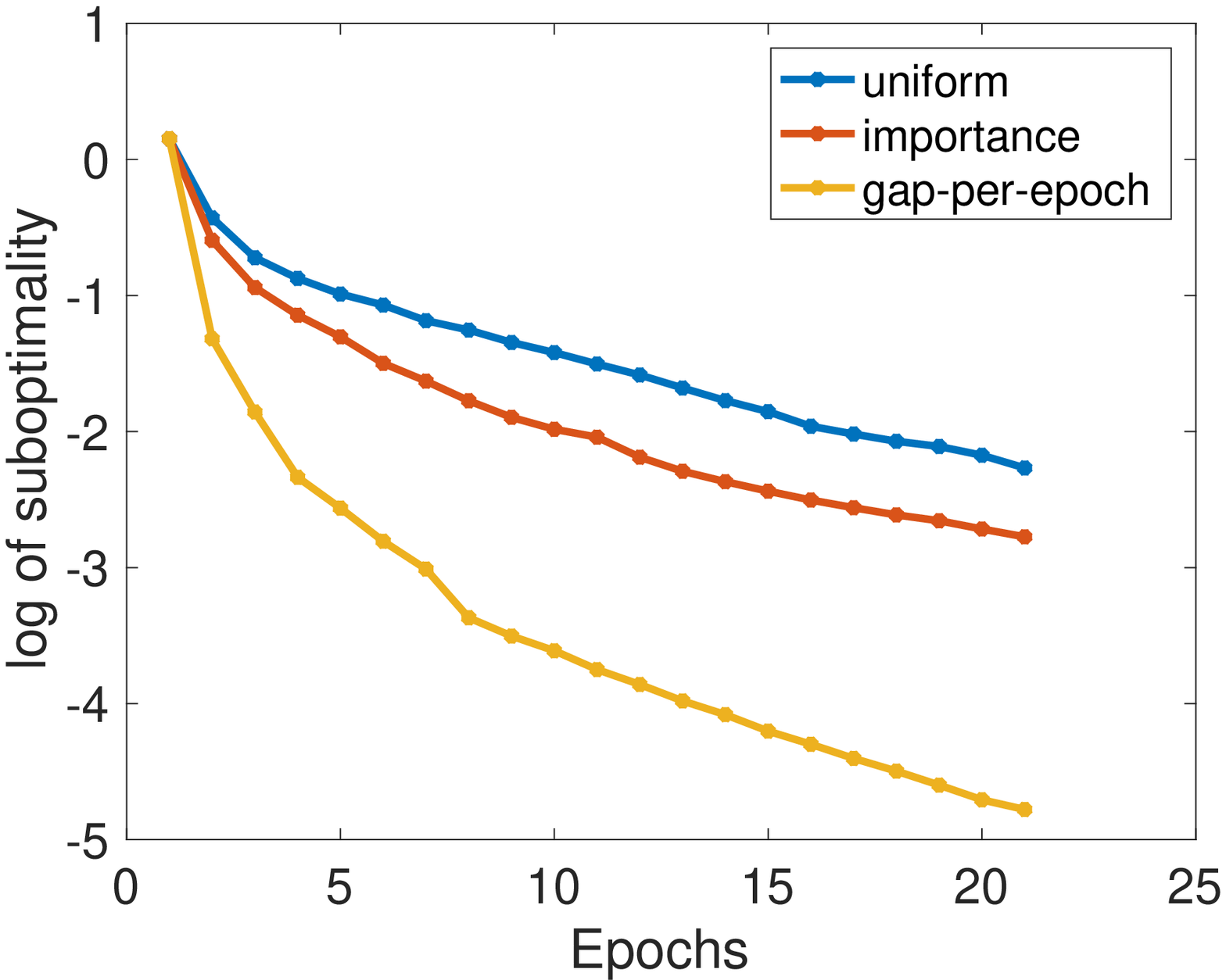}
  \caption{Subopt., fixed distr.}
\end{subfigure}%
\begin{subfigure}{.24\textwidth}
  \centering
\includegraphics[width=0.9\linewidth]{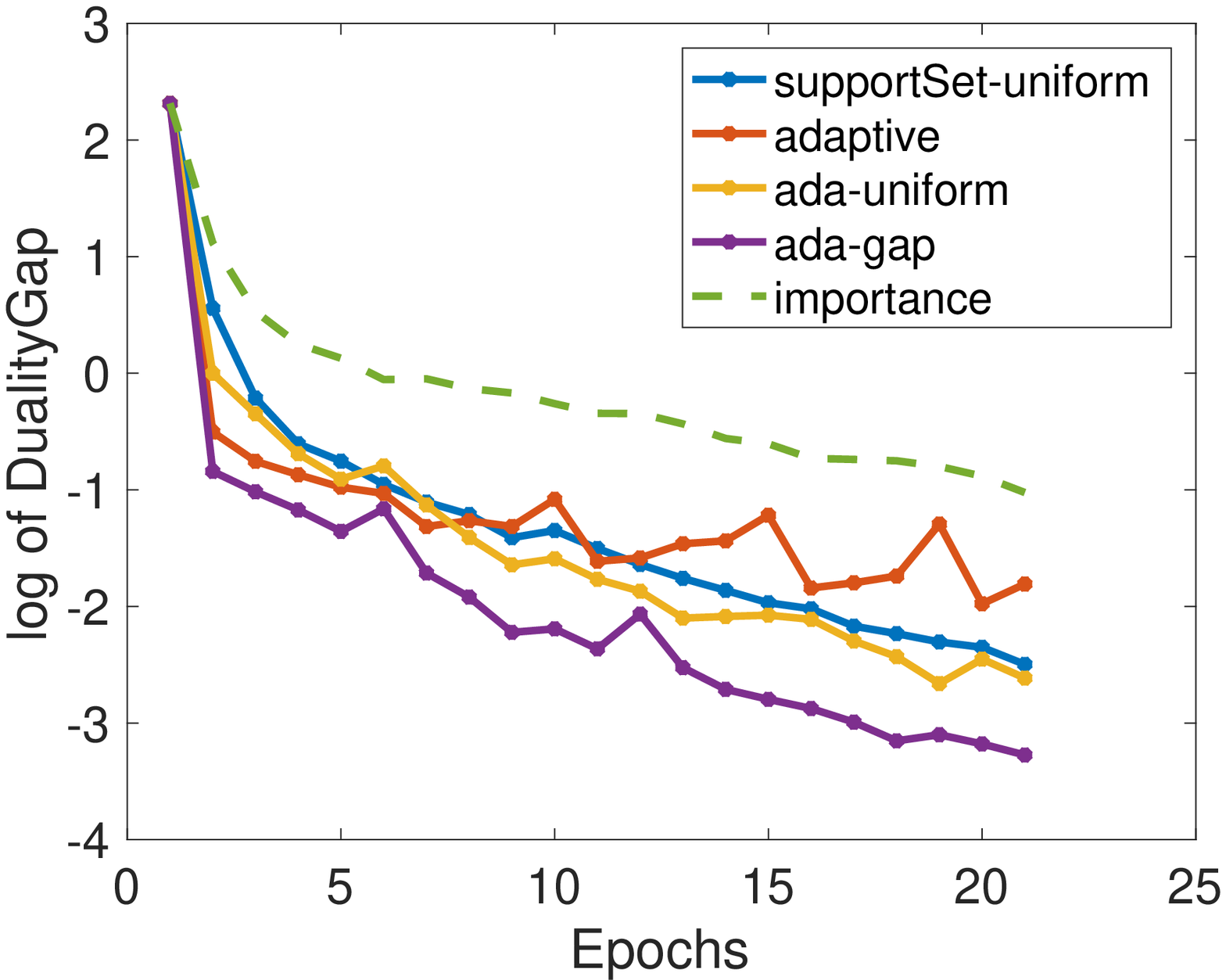}
  \caption{Gap, adaptive}
\end{subfigure}%
\begin{subfigure}{.24\textwidth}
  \centering
\includegraphics[width=0.9\linewidth]{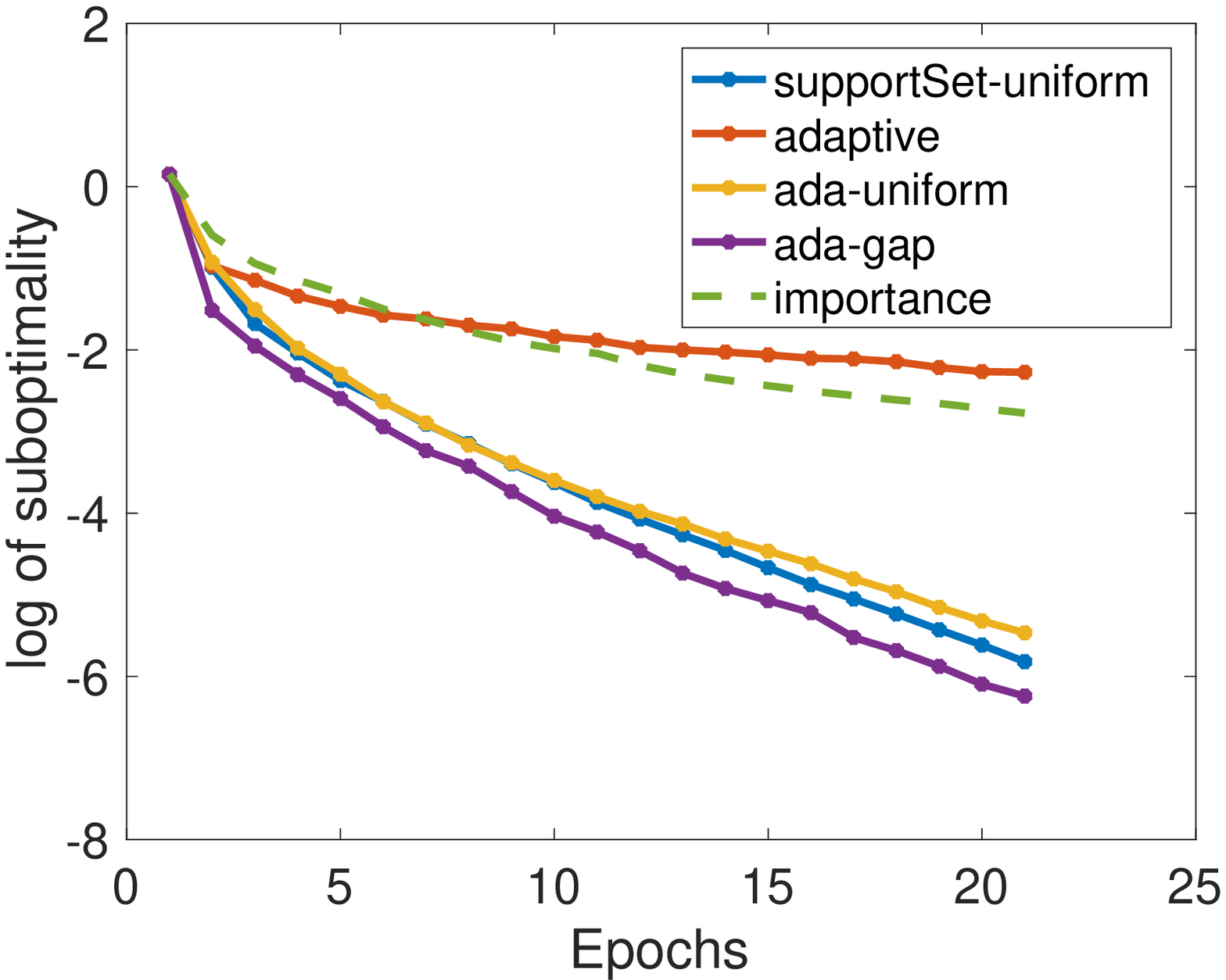}
  \caption{Suboptimality, adaptive}
\end{subfigure}%
\\
\rotatebox[origin=t]{90}{SVM, \textit{\small ionosphere}}
\begin{subfigure}{.24\textwidth}
  \centering
\includegraphics[width=0.9\linewidth]{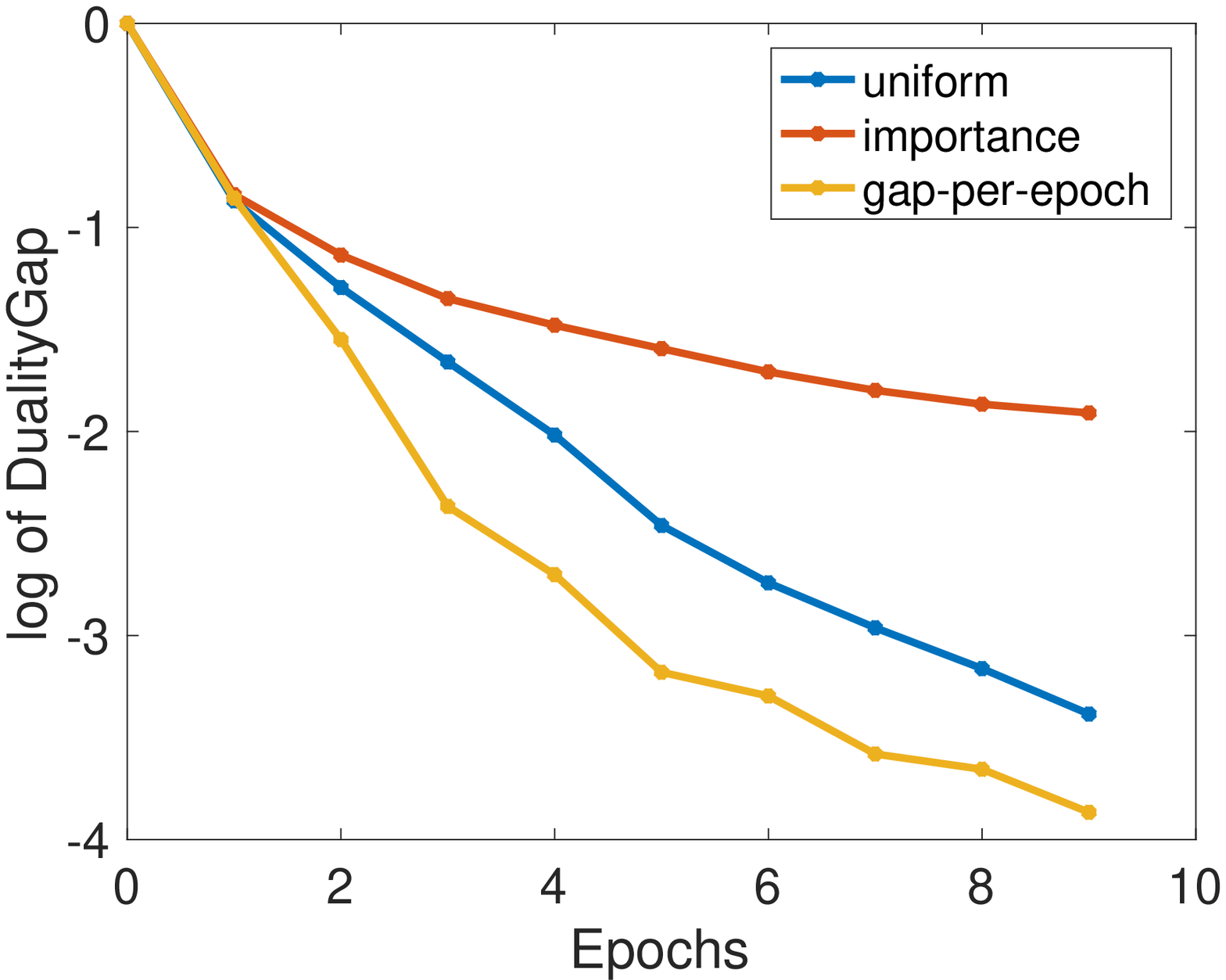}
  \caption{Gap, fixed distr.}
\end{subfigure}%
\begin{subfigure}{.24\textwidth}
  \centering
\includegraphics[width=0.9\linewidth]{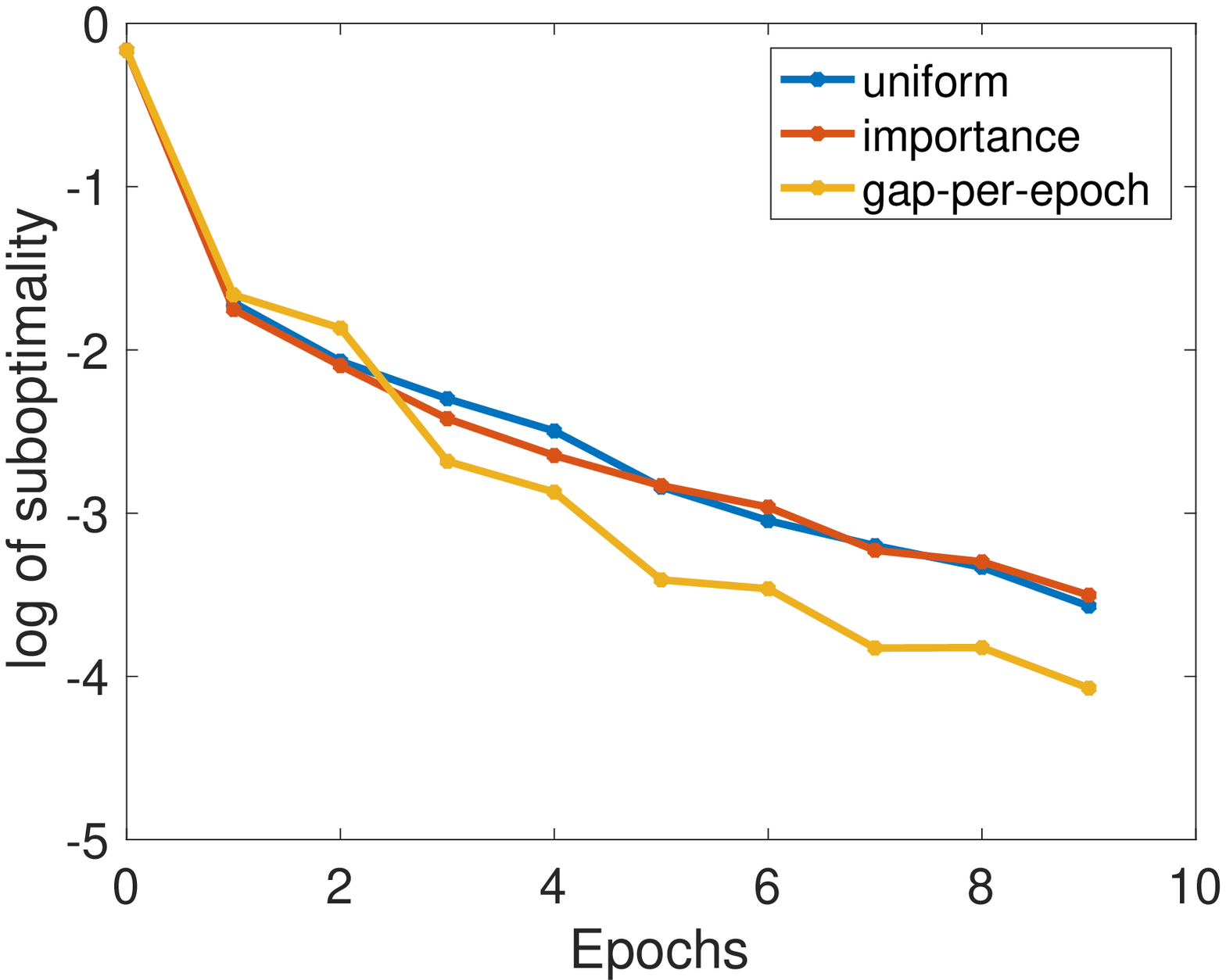}
  \caption{Subopt., fixed distr.}
\end{subfigure}%
\begin{subfigure}{.24\textwidth}
  \centering
\includegraphics[width=0.9\linewidth]{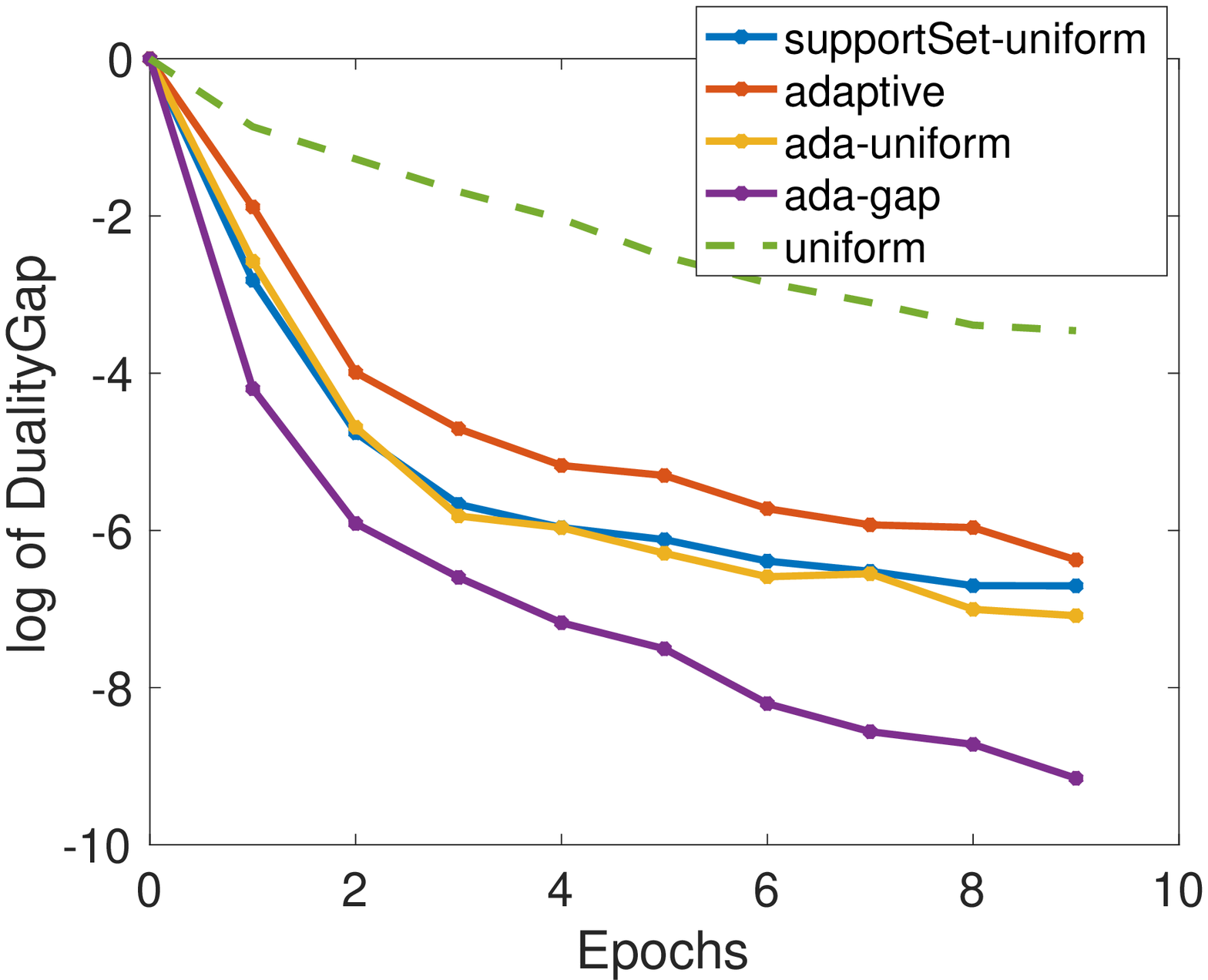}
  \caption{Gap, adaptive}
\end{subfigure}%
\begin{subfigure}{.24\textwidth}
  \centering
\includegraphics[width=0.9\linewidth]{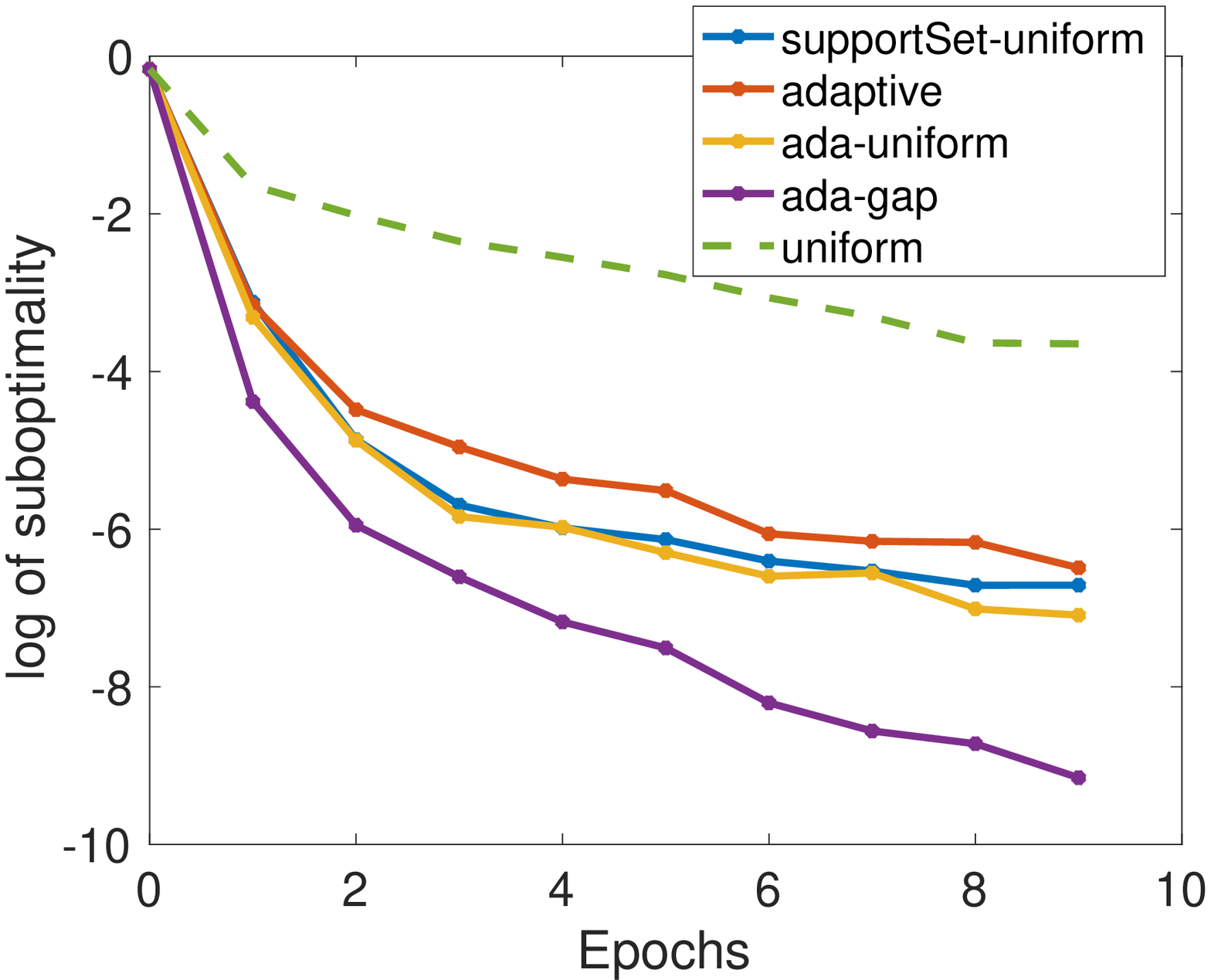}
  \caption{Suboptimality, adaptive}
\end{subfigure}%
\caption{Lasso (first two rows) and SVM (bottom row). Comparison of different fixed and adaptive variants of CD, reporting duality gap and suboptimality measures vs. epochs - rcv1*, mushrooms and ionosphere datasets.}
\label{fig:results}
\end{figure*}

\paragraph{Total cost per epoch}
In a naive implementation, the most expensive sampling schemes are \textit{adaptive, supportSet-uniform, ada-uniform} and \textit{ada-gap}. Those completely recompute the sampling distribution after each iteration, giving a total per-epoch complexity of $\mathcal O(n \cdot \text{nnz})$. %
In contrast, the fixed non-uniform sampling scheme \textit{importance} requires to build the sampling distribution only once, or once per epoch for \textit{gap-per-epoch} (both giving $\mathcal O(\text{nnz})$ operations). The complexity of $n$ samplings using the tree structure is $\mathcal O(n \log(n))$, the complexity of a variable update is $\mathcal O(\text{nnz})$. Overall, the asymptotic complexity therefore is $\mathcal O(n \log(n) + \text{nnz})$ per epoch, compared to $\mathcal O(\text{nnz})$ for simple uniform sampling.

\begin{figure}
\centering
\begin{subfigure}{.25\textwidth}
  \centering
\includegraphics[width=0.9\linewidth]{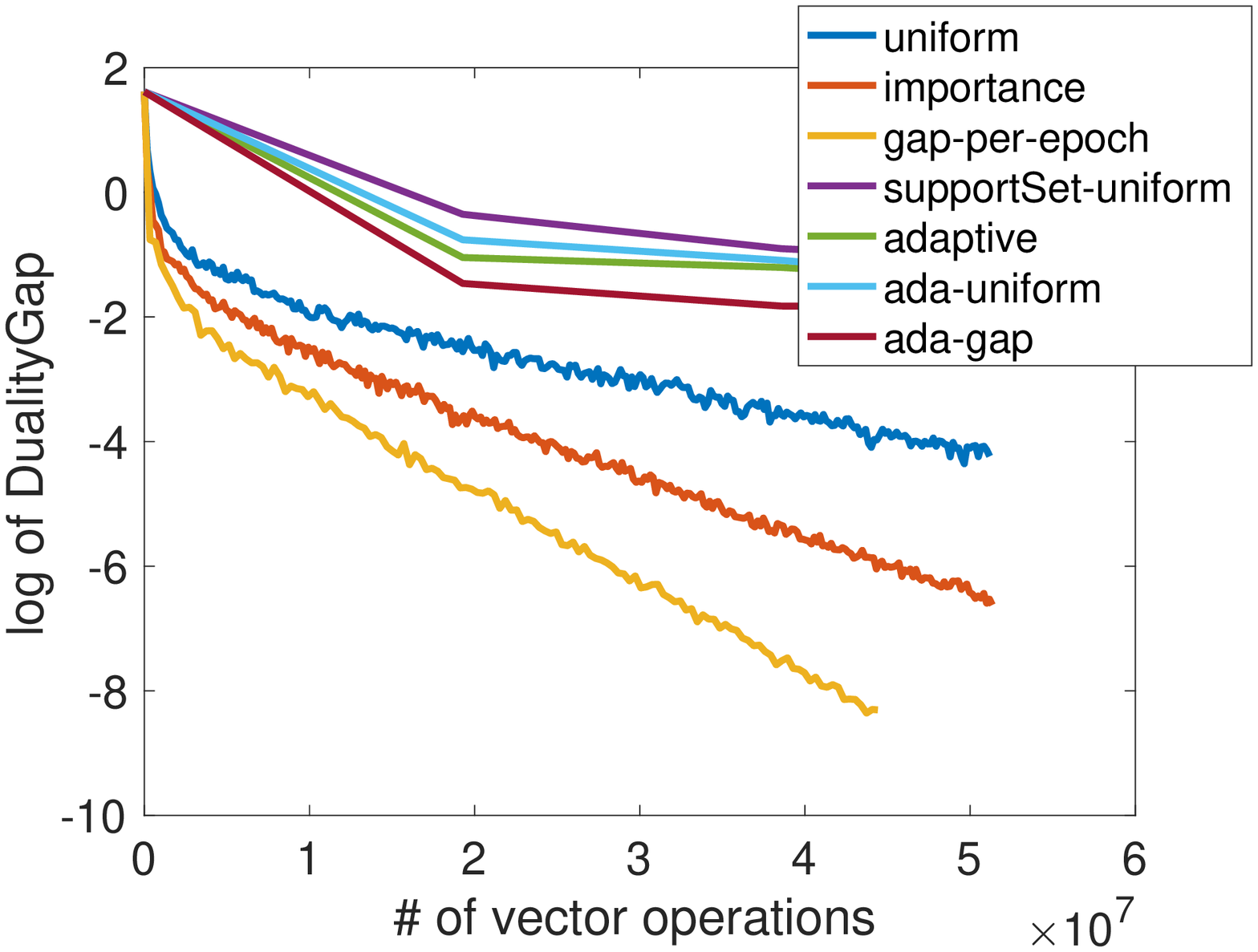}
  \caption{Gap}
\end{subfigure}%
\begin{subfigure}{.25\textwidth}
  \centering
\includegraphics[width=0.9\linewidth]{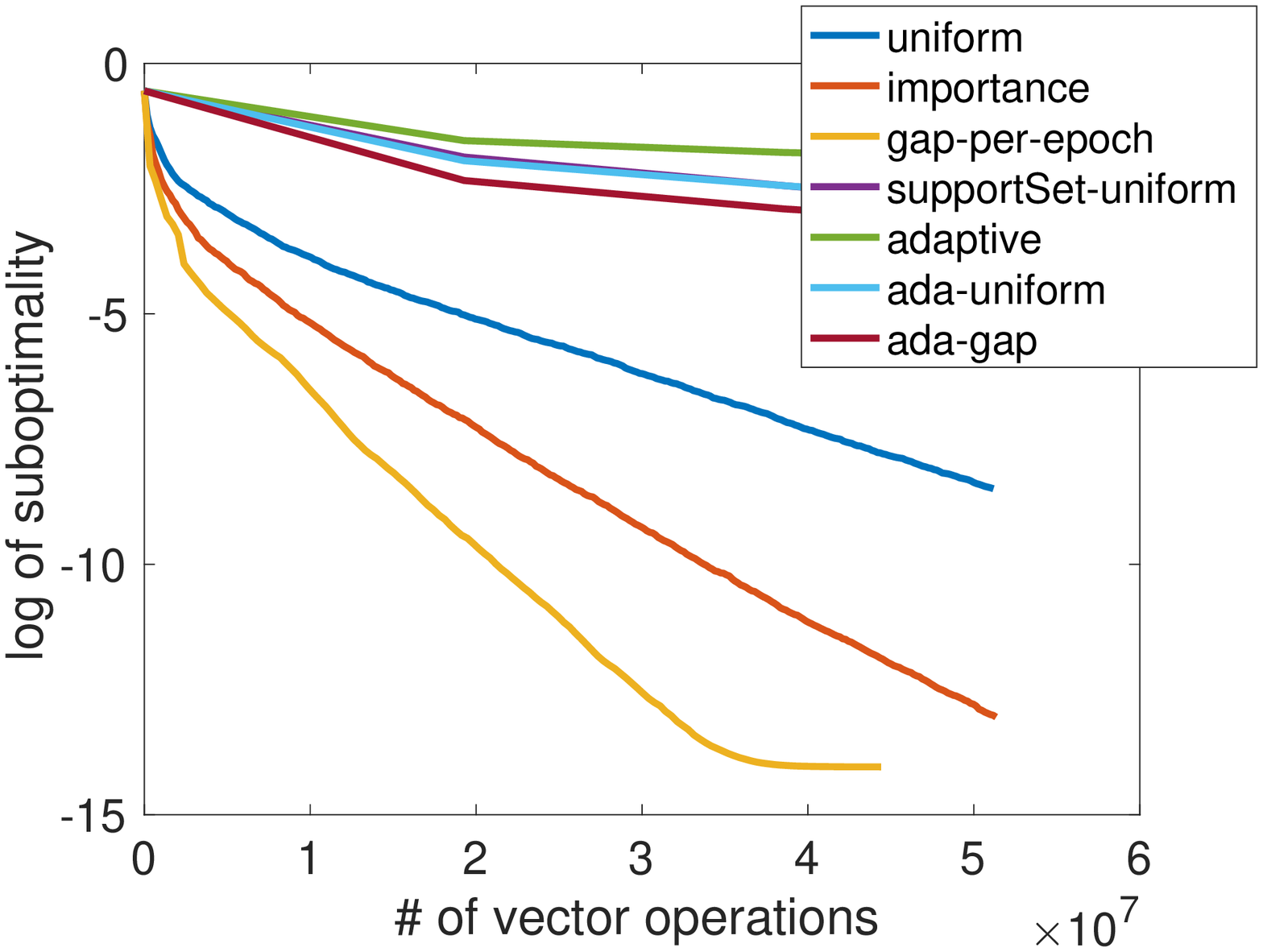}
  \caption{Suboptimality}
\end{subfigure}%
\caption{Lasso on the \textit{mushrooms} dataset. Performance in terms of duality gap and suboptimality, plotted against the total number of vector operations.}
\label{fig:resultsVecOp}
\end{figure}

\vspace{-3mm}
\section{Experimental results}
We provide numerical evidence for our CD sampling strategies on two key machine learning problems: The Lasso and the hinge-loss SVM.
All our algorithms and theory are also directly applicable to sparse logistic regression and others, but we omit experiments due to space limitations.\vspace{-2mm}

\paragraph{Datasets.}
The experiments are performed on three standard datasets listed in Table \ref{data_lasso}, available\footnote{\href{https://www.csie.ntu.edu.tw/~cjlin/libsvmtools/datasets/}{www.csie.ntu.edu.tw/~cjlin/libsvmtools/datasets/}} from the UCI repository \cite{Asuncion+Newman:2007}. %
Note that \textit{rcv1*} is a randomly subsampled\footnote{We randomly picked 10000 datapoints and 1000 features, and then removed zero rows and columns.}
 version the \textit{rcv1} dataset. Experiments on the full \textit{rcv1} dataset are provided in Appendix \ref{app:expBIG}. 
\vspace{-3mm}
\begin{table}[H]
\caption{Datasets}\vspace{-2mm}
\label{data_lasso}
{\small
\begin{tabular}{|l| l| l| l| l|}%
\hline
Dataset & $d$ & $n$ & nnz$/(nd)$ & $c_v = \frac{\mu(\|\av_i\|)}{\sigma(\|\av_i\|)}$ \tabularnewline %
\hline
mushrooms & 112 & 8124 & 18.8\% & 1.34 \tabularnewline %
\hline
rcv1* & 809 & 7438 &0.3\% & 0.62 \tabularnewline %
\hline
ionosphere & 351 & 33 & 88\% & 3.07\tabularnewline
\hline 
\end{tabular}
}%
\end{table}

\paragraph{Setup.}
For Lasso, the regularization parameter $\lambda$ in \eqref{lasso_problem} is set such that the cardinality of the true support set is between 10\%  and 15 \% of the total number of features $n$. We use $\lambda = 0.05$ for \textit{mushrooms}, and $\lambda = 7 \cdot 10^{-4}$  for  \textit{rcv1*}.
For hinge-loss SVM, the regularization parameter $\lambda$ is chosen such that the classification error on the test set was comparable to training error. We use $\lambda = 0.1$ for \textit{ionosphere}. %

\paragraph{Performance.}
Figures \ref{fig:results} and \ref{fig:resultsVecOp} show the performance of all our studied variants of CD.
We record \emph{suboptimality} and \emph{duality gap} (see \eqref{duality_gap}) as the main measures of algorithm performance. %
All reported results are averaged over 5 runs of each algorithm.

\paragraph{Methods with fixed sampling distributions.}
For the three efficient sampling schemes, our results %
show that \textsf{CD importance} converges faster than \textsf{CD uniform} on both datasets for Lasso, however it is worse than the uniform on SVM. The ``mildly'' adaptive strategy \textsf{gap-per-epoch}, based on our coordinate-wise duality gap theory but computed only per-epoch, significantly outperforms both of them. This is observed both in number of epochs (Figure \ref{fig:results}) as well as number of total vector operations (Figure \ref{fig:resultsVecOp}).

\paragraph{Methods with adaptive sampling distributions}
For the adaptive methods updating the probabilities after each coordinate step in Figure \ref{fig:results}, we show importance sampling as a baseline method (dashed line). We see that measured per epoch, all adaptive methods outperform the fixed sampling methods.%
Among all adaptive methods, the \textsf{ada-gap} algorithm shows better convergence speed with both suboptimality and duality gap measures.

\paragraph{Highlights.\vspace{-3mm}}
\begin{itemize}
\item The experiments for Lasso show a clear advantage of non-uniform sampling over the uniform, and superiority of the adaptive sampling over the fixed non-uniform, which is supported by our theory. %
\vspace{-1mm}
\item Among the adaptive methods per iteration, the best performance for both Lasso and SVM in terms of epochs is by \textsf{ada-gap}, which has proven convergence bounds (Theorem~\ref{theorem_gapwise}), but also has high computational cost ($\Theta(d \cdot \text{nnz})$).
\vspace{-1mm}
\item The best sampling scheme in terms of total computational cost is \textsf{gap-per-epoch}, which is the epoch-wise variant of the \textsf{ada-gap} algorithm (based on recomputing duality gaps once per epoch), as shown in Figure \ref{fig:resultsVecOp}. %
\vspace{-1mm}
\end{itemize}

\section{A discussion on the results}
\vspace{-1mm}
Coordinate descent methods have a rich history in the discipline of optimization as well as many machine learning applications, \textit{cf.}, \cite{Wright:2015bn} for a review. %

For SVMs, CD related methods have been studied since their introduction, e.g., by \cite{Friess:1998te}. \citet{Hsieh:2008bd} is the first to propose CD in the partially separable primal-dual setting for hinge-loss. Theoretical convergence rates beyond the application of the hinge-loss can be found in the SDCA line of work \cite{ShalevShwartz:2013wl}, which is the primal-dual analog of the primal-only SGD algorithms. However, the main limitation of SDCA is that it is only applicable to strongly convex regularizers, or requires smoothing techniques \cite{Nesterov:2005ic} to be applicable to general regularizers such as L1. The technique of \citet{Dunner:2016vga} can extend to the CD algorithms as well as the primal-dual analysis to the problem class of interest here, using a bounded set of interest for the iterates instead of relying on smoothing. 

The convergence rate of stochastic methods (such as CD and SGD) naturally depends on a sampling probability distribution over the coordinates or datapoints respectively. While virtually all existing methods use sampling uniformly at random \cite{Hsieh:2008bd,ShalevShwartz:2011vo,LacosteJulien:2013ue,ShalevShwartz:2012tn},  recently \cite{Nesterov:2012fa,Qu:2016bd,Zhao:2014vg,AllenZhu:2016wj} showed that an appropriately defined fixed non-uniform sampling distribution, dubbed as \textit{importance sampling}, can significantly improve the convergence.

The work of \cite{Csiba:2015ue} has taken the non-uniform sampling a step further towards adaptive sampling. %
While restricted to strongly convex regularizers, the rates provided for the AdaSDCA algorithm of \cite{Csiba:2015ue} - when updating all probabilities after each step - can beat the ones for uniform and \textit{importance sampling}. A different approach for adapting the sampling distribution is proposed in \cite{Osokin:2016tp}, where the block coordinate Frank-Wolfe algorithm is enhanced with sampling proportional to values of block-wise duality gaps. An adaptive variant of SGD is studied by \cite{Papa:2015gv}, where they proposed an adaptive sampling scheme dependent on the past iterations in a Markovian manner, without giving explicit convergence rates. Other adaptive heuristics without proven convergence guarantees include ACF \cite{Glasmachers:2014wu} and ACiD~\cite{loshchilov:2011}.
\vspace{-1mm}

For general convex regularizers such as L1, the development of CD algorithms includes
\cite{Fu:1998cd,Friedman:2007ut,Friedman:2010wm} and more recent extensions also improving the theoretical convergence rates \cite{ShalevShwartz:2011vo,Johnson:2015tq,Zhao:2014wy}. All are restricted to uniform sampling, and we are not aware of proven convergence rates showing improvements of non-uniform or even adaptive sampling for unmodified L1 problems. \cite{Zhao:2014vg,AllenZhu:2016wj} show improved rates for non-uniform sampling for L1 but require a smoothing modification of the original problem, and are not covering adaptive sampling.
\vspace{-1mm}

\paragraph{Conclusion.}
In this work, we have investigated \textit{adaptive} rules for adjusting the sampling probabilities in coordinate descent. 
Our theoretical results provide improved convergence rates for a more general class of algorithm schemes on one hand, and optimization problems on the other hand, where we are able to directly analyze CD on general convex objectives (as opposed to strongly convex regularizers in previous works). This is particularly useful for L1 problems and (original) hinge-loss objectives, which were not covered by previous schemes.
Our practical experiments confirm the strong performance of the adaptive algorithms, and confirm that the behavior predicted by our theory. Finally, we advocate the use of the computationally efficient \textsf{gap-per-epoch} sampling scheme in practice. While the scheme is close to the ones supported by our theory, an explicit primal-dual convergence analysis remains a future research question.
\vspace{-1mm}

\paragraph{Acknowledgements.}
We thank Dominik Csiba and Peter Richt{\'a}rik for helpful discussions.
VC was supported in part by the European Commission under Grant ERC Future Proof, SNF 200021-146750, and SNF CRSII2-147633.

{\small
\bibliographystyle{plainnat}
\bibliography{bibliography}
}

%\end{document}

\clearpage
\appendix
\section{Proofs}
\label{app:proof}

\begin{replemma}{lemma1}
Let $f$ be $1/\beta$-smooth and each $g_i$ be $\mu_i$-strongly convex with convexity parameter $\mu_i \geq 0$ $\forall i \in [n]$. For the case $\mu_i = 0$, we require $g_i$ to have a bounded support. Then for any iteration $t$, any sampling distribution $\pv^{(t)}$ and any arbitrary $s_i \in [0,1]$ $\forall i \in [n]$, the iterates of the CD method satisfy 
\begin{equation}
\label{eq1Again}
\begin{aligned}
&\E[\OA(\alphav^{(t+1)}) | \alphav^{(t)}] \leq \OA(\alphav^{(t)}) - \sum_{i} s_i p_i^{(t)} G_i(\alphav^{(t)}) \\&-\sum_{i} p_i^{(t)} \Big ( \frac{\mu_i(s_i-s_i^2)}{2} - \frac{s_i^2 \|\av_i\|^2}{2 \beta}  \Big ) |\kappa_i^{(t)}|^2,
\end{aligned}
\end{equation}
here $\kappa_i^{(t)}$ is $i$-th dual residual.
\end{replemma}
\begin{proof}
Since in CD update ($\alphav^{(t+1)} = \alphav^{(t)} + \ev_i \Delta \alpha_i$) only one coordinate per iteration is changed, the one iteration improvement in dual objective can be written as:
$$\begin{aligned}
&\OA(\alphav^{(t)}) - \OA(\alphav^{(t+1)}) \\&= \underbrace{\Big [ g_i(\alpha_i^{(t)}) +  f(A\alphav^{(t)}) \Big ]}_{(A)} - \underbrace{\Big [ g_i(\alpha_i^{(t+1)}) +  f(A\alphav^{(t+1)}) \Big ]}_{(B)}.
\end{aligned}$$
To bound part $(B)$ we use a suboptimal update $\Delta \alpha_i = s_i \kappa_i^{(t)}$, for all $s_i \in [0,1]$:
$$\begin{aligned}
(B) &=   g_i(\alpha_i^{(t+1)}) +  f(A\alphav^{(t+1)}) \\ &\leq \min_{\Delta \alpha_i} \Big [ g_i(\alpha_i^{(t)} + \Delta \alpha_i) +  f(A\alphav^{(t)} + \av_i \Delta \alpha_i) \Big ]\\ &\leq g_i(\alpha_i^{(t)} + s_i \kappa_i^{(t)}) +  f(A\alphav^{(t)} + \av_i s_i \kappa_i^{(t)}). 
\end{aligned}$$
Each of $g_i$ is $\mu_i$-strongly convex, therefore:
$$\begin{aligned}
&g_i(\alpha_i^{(t)} + s_i \kappa_i^{(t)}) = g_i(\alpha_i^{(t)} + s_i (u_i^{(t)} - \alpha_i^{(t)}))\\ &= g_i(s_i(u_i^{(t)}) + (1-s_i)(\alpha_i^{(t)})) 
\\ &\leq s_i g_i(u_i^{(t)}) + (1-s_i)g_i(\alpha_i^{(t)}) - \frac{\mu_i}{2}s_i(1-s_i)(\kappa_i^{(t)})^2.
\end{aligned}$$ 
The function $f$ is $\frac{1}{\beta}$-smooth:
$$\begin{aligned}
&f(A\alphav^{(t)} + \av_i s_i \kappa_i^{(t)}) \\&\leq f(A\alphav^{(t)}) + \nabla f(A\alphav^{(t)})^\top (s_i \kappa_i^{(t)} \av_i) + \frac{1}{2\beta} \| s_i \kappa_i^{(t)} \av_i\|^2.
\end{aligned}$$
As a result:
$$\begin{aligned}
(B) &\leq  s_i g_i(u_i^{(t)}) - s_i g_i(\alpha_i^{(t)}) - \frac{\mu_i}{2}s_i(1-s_i)(\kappa_i^{(t)})^2 \\
 & + \underbrace{g_i(\alpha_i^{(t)}) + f(A\alphav^{(t)})}_{(A)} + \nabla f(A\alphav^{(t)})^\top (s_i \kappa_i^{(t)} \av_i) \\&+ \frac{1}{2\beta} \| s_i \kappa_i^{(t)} \av_i\|^2. 
\end{aligned}$$
With obtained results above and optimality condition $\wv(\alphav) = \nabla f(A\alphav)$, the improvement in dual objective can be written as:
$$ \begin{aligned} &\OA(\alphav^{(t)}) - \OA(\alphav^{(t+1)}) = (A) - (B)  \\
 & \geq - s_i g_i(u_i^{(t)}) + s_i g_i(\alpha_i^{(t)}) + \frac{\mu_i}{2}s_i(1-s_i)(\kappa_i^{(t)})^2 \\
 & - \wv(\alphav^{(t)}) (s_i u_i^{(t)} \av_i)  + \wv(\alphav^{(t)}) (s_i \alpha_i^{(t)} \av_i) - \frac{1}{2\beta} \| s_i \kappa_i^{(t)} \av_i\|^2\\
 & = s_i \Big(-g_i(u_i^{(t)}) + g_i(\alpha_i^{(t)}) - \wv(\alphav^{(t)}) (u_i^{t-1} \av_i)   \\
 & + \wv(\alphav^{(t)}) (\alpha_i^{t-1} \av_i) + \frac{\mu_i}{2}(1-s_i)(\kappa_i^{(t)})^2 - \frac{s_i}{2\beta} \|\av_i \|^2  |\kappa_i^{(t)}|^2 \Big).
\end{aligned}$$
Since $u_i^{(t)} \in \partial g_i^*(-\av_i^\top\wv(\alphav^{(t)}))$, the Fenchel-Young inequality %
becomes equality for $g_i(u_i^{(t)})$:
$$g_i(u_i^{(t)}) + g_i^*(-\av_i^\top\wv(\alphav^{(t)})) = -\wv(\alphav^{(t)}) (u_i^{(t)} \av_i) $$ 
Using this fact, the bound on the improvement in dual objective becomes:
$$\begin{aligned} &\OA(\alphav^{(t)}) - \OA(\alphav^{(t+1)}) \geq s_i \Big( g_i(\alpha_i^{(t)}) + g_i^*(-\av_i^\top\wv(\alphav^{(t)})) \\& + \wv(\alphav^{(t)}) (\alpha_i^{(t)} \av_i) + \frac{\mu_i}{2}(1-s_i)(\kappa_i^{(t)})^2 - \frac{s_i}{2\beta} \|\av_i \|^2  |\kappa_i^{(t)} |^2 \Big)
\end{aligned}$$
Therefore for any $s_i \in [0,1]$ it holds that:
\begin{equation}
\label{oneCoordinate}
\begin{aligned}
&\OA(\alphav^{(t)}) - \OA(\alphav^{(t+1)}) \\&\geq s_i \Big [ G_i(\alphav^{(t)}) + \frac{\mu_i}{2}(1-s_i)|\kappa_i^{(t)}|^2 - \frac{s_i}{2\beta} \|\av_i \|^2  |\kappa_i^{(t)}|^2  \Big ],
\end{aligned}
\end{equation}
where $G_i$ is $i$-th coordinate-wise duality gap:
$$\begin{aligned} 
&G(\alphav^{(t)})  = \sum_i G_i(\alphav^{(t)}) \\&G_i(\alphav^{(t)}) =  g_i^*(-\av_i^\top \wv) + g_i(\alpha_i^{(t)}) + \alpha_i^{(t)}\av_i^\top\wv.
\end{aligned}$$
By taking an expectation of the both sides with respect to $i$, conditioned on $\alphav^{(t)}$, we obtain:
$$\begin{aligned}
&\E[\OA(\alphav^{(t+1)}) | \alphav^{(t)}] \leq \OA(\alphav^{(t)}) - \sum_{i} s_i p_i^{(t)} G_i(\alphav^{(t)}) \\&- \sum_{i} p_i^{(t)} \Big ( \frac{\mu_i(s_i-s_i^2)}{2} - \frac{s_i^2 \|\av_i\|^2}{2 \beta}  \Big ) |\kappa_i^{(t)}|^2
\end{aligned}$$
and thus finalize the proof.
\end{proof}

\newpage
\begin{reptheorem}{thm:main}
Assume $f$ is $\frac{1}{\beta}$-smooth function. Then, if $g^*_i$ is $L_i$-Lipschitz for each $i$ and $\pv^{(t)}$ is coherent with $\kappav^{(t)}$, then the iterates of the CD method satisfies  
\begin{equation}
\label{eq:bound1Again}
\E[\varepsilon_A^{(t)}] \leq \frac{2F^{\circ}n^2 + \frac{2\varepsilon_A^{(0)}}{p_{\min}}}{\frac{2}{p_{\min}} + t}.
\end{equation}
Moreover, the overall number of iterations $T$ to obtain a duality gap $G(\bar{\alphav}) \leq \varepsilon$ must satisfy the following: 
\begin{equation}
\label{eq:theorem1Again}
T \geq \max\Bigg \{0, \frac{1}{p_{\min}} \log \Big (\frac{2 \varepsilon_A^{(0)}}{n^2 p_{\min} F^{\circ}} \Big) \Bigg\} + \frac{5F^{\circ}n^2}{\varepsilon} - \frac{1}{p_{\min}}.
\end{equation}
Moreover, when $t \geq T_0$ with 
\begin{equation}
\label{eq:theorem2Again}
T_0 = \max\Bigg \{0, \frac{1}{p_{\min}} \log \Big (\frac{2 \varepsilon_A^{(0)}}{n^2 p_{\min} F^{\circ}} \Big) \Bigg\}  +  \frac{4 F^{\circ} n^2}{\varepsilon} - \frac{2}{p_{\min}}
\end{equation}
we have the suboptimality bound of  $\E[\OA(\alphav^{(t)}) - \OA(\alphav^{\star})]  \leq \varepsilon/2$, where $\varepsilon^{(0)}_A$ is the initial dual suboptimality and $F^{\circ}$ is an upper bound on $\E[F^{(t)}]$ taken over the random choice of the sampled coordinate at $1,\dots,T_0$ algorithm iterations.
\end{reptheorem}
\begin{proof}
According to Lemma \ref{duality_lip-bound} if $g^*_i$ is $L_i$-Lipschitz, then $g_i$ has $L_i$-bounded support and the conditions of Remark \ref{remark1} are satisfied. From Remark \ref{remark1} we know:
\begin{equation} \label{lemma}
\E[\OA(\alphav^{(t+1)})|\alphav^{(t)}] \leq \OA(\alphav^{(t)}) - \theta G(\alphav^{(t)}) + \frac{\theta^2 n^2}{2} F^{(t)}
\end{equation}
With $\OA(\alphav^{(t)}) - \OA(\alphav^{(t+1)}) = \varepsilon_A^{(t)} - \varepsilon_A^{(t+1)}$ and $\varepsilon_A^{(t)} = \OA(\alphav^{(t)}) - \OA(\alphav^{*}) \leq  G(\alphav^{(t)}),$ this implies:
$$\E[\varepsilon_A^{(t+1)}|\alphav^{(t)}] \geq  \varepsilon_A^{(t)} - \theta \varepsilon_A^{(t)} + \frac{\theta^2 n^2}{2} F^{(t)}$$
by taking unconditional expectation over all iterations and using definition of $F^{\circ}$ we obtain:
$$\begin{aligned}
\E[\varepsilon_A^{(t+1)}]  &\leq (1 - \theta) \E[\varepsilon_A^{(t)}] + \frac{\theta^2 n^2}{2} \E[F^{(t)}] \\&\leq (1 - \theta) \E[\varepsilon_A^{(t)}] + \frac{\theta^2 n^2}{2} F^{\circ}
\end{aligned}
$$
Now we will show using induction that we can bound the dual suboptimality as:
\begin{equation} \label{suboptimality_bound}
\E[\varepsilon_A^{(t)}] \leq \frac{2F^{\circ}n^2}{\frac{2}{p_{\min}} + t - t_0}, 
\end{equation}
where $t \geq t_0 =  \max\Big \{0, \frac{1}{p_{\min}} \log \big (\frac{2 \varepsilon_A^0}{n^2 p_{\min} F^{\circ}} \big) \Big\}$. Indeed, let's choose $\theta = p_{\min}$, then the basis of induction at $t=t_0$ is verified as:
$$\begin{aligned} \E[\varepsilon_A^{(t)}]  &\leq (1-p_{\min})^t \varepsilon_A^{(0)} + \sum_{i=0}^{t-1} (1-p_{\min})^i p_{\min}^2 n^2 \frac{F^{\circ}}{2} \\ &\leq e^{-t p_{\min}} \varepsilon_A^{(0)} + n^2 p_{\min} \frac{F^{\circ}}{2} \\ &\leq n^2 p_{\min} F^{\circ}. \end{aligned}$$
Note that if in \eqref{suboptimality_bound} instead of $F^{\circ}$ we take $F'^{\circ}:= F^{\circ} + \frac{\varepsilon_A^{(0)}}{n^2p_{\min}},$ the condition holds with $t_0=0$:
$$\E[\varepsilon_A^{(t)}] \leq \frac{2F'^{\circ}n^2}{\frac{2}{p_{\min}} + t} =  \frac{2F^{\circ}n^2 + \frac{2\varepsilon_A^{(0)}}{p_{\min}}}{\frac{2}{p_{\min}} + t}.$$
Now let's prove the inductive step, for $t > t_0$. Suppose claim holds for $t-1$, then 
$$\begin{aligned} \E[\varepsilon_A^{(t)}]  &\leq (1-\theta) \E[\varepsilon_A^{(t-1)}] +  \theta^2 n^2 \frac{F^{\circ}}{2} \\ &\leq (1-\theta)  \frac{2F^{\circ}n^2}{\frac{2}{p_{\min}} + (t - 1) - t_0}  + \theta^2  n^2 \frac{F^{\circ}}{2} ,\end{aligned}$$
choosing $\theta = \frac{2}{\frac{2}{p_{\min}}+t -1 - t_0} \leq p_{\min} $ yeilds:
$$\begin{aligned} \E[\varepsilon_A^{(t)}]  &\leq \Bigg (1- \frac{2}{\frac{2}{p_{\min}}+t -1 - t_0} \Bigg ) \frac{2F^{\circ}n^2}{\frac{2}{p_{\min}} + (t - 1) - t_0} \\&+ \Bigg (\frac{2}{\frac{2}{p_{\min}}+t -1 - t_0} \Bigg )^2 \frac{F^{\circ} n^2}{2} \\
&=\Bigg (1- \frac{2}{\frac{2}{p_{\min}}+t -1 - t_0} \Bigg ) \frac{2F^{\circ}n^2}{\frac{2}{p_{\min}} + (t - 1) - t_0}  \\&+ \Bigg (\frac{1}{\frac{2}{p_{\min}}+t -1 - t_0}  \Bigg )  \frac{2F^{\circ}n^2}{\frac{2}{p_{\min}} + (t - 1) - t_0}  \\
&= \Bigg (1- \frac{1}{\frac{2}{p_{\min}}+t -1 - t_0} \Bigg ) \frac{2F^{\circ}n^2}{\frac{2}{p_{\min}} + (t - 1) - t_0} \\
&= \frac{2F^{\circ}n^2}{\frac{2}{p_{\min}} + (t - 1) - t_0} \Bigg (\frac{\frac{2}{p_{\min}}+t -2 - t_0}{\frac{2}{p_{\min}}+t -1 - t_0} \Bigg ) \\
&\leq \frac{2F^{\circ}n^2}{\frac{2}{p_{\min}} + t - t_0}.
\end{aligned}$$
This proves the bound \eqref{suboptimality_bound} on suboptimality. To bound the duality gap we sum the inequality \eqref{lemma} over the interval $t=T_0+1, ..., T$ and obtain 
$$\begin{aligned}
&\E[\OA(\alphav^{(T_0)}) - \OA(\alphav^{(T)})] \\&\geq \theta \E \Big [\sum_{t=T_0+1}^{T} \OA(\alphav^{(t-1)}) + \OB(\wv^{(t-1)}) \Big ] \\&- (T-T_0) \frac{\theta^2 n^2}{2} F^{\circ}, 
\end{aligned}$$
by rearranging terms and choosing $\bar{\wv}$ and $\bar{\alphav}$ to be the average vectors over $t \in \{ T_0,T-1\}$ we get:
$$\begin{aligned}
&\E[G(\bar{\alphav})]= \E[\OA(\bar{\alphav}) + \OB(\bar{\wv})] \\&\leq \frac{\E[\OA(\alphav^{(T_0)}) - \OA(\alphav^{(T)})]}{\theta (T-T_0)} + \theta n^2 \frac{F^{\circ}}{2}.
\end{aligned}$$
If $T \geq \frac{1}{p_{\min}} + T_0$ and $T_0 \geq t_0$, we can set $\theta = 1/(T-T_0)$ and combining this with \eqref{suboptimality_bound} we get:
$$\begin{aligned} \E[G(\bar{\alphav})] &\leq \E[\OA(\alphav^{(T_0)}) - \OA(\alphav^{(T)})] +  \frac{F^{\circ} n^2}{2 (T-T_0)}\\
&\leq \E[\OA(\alphav^{(T_0)}) - \OA(\alphav^{*})] +  \frac{F^{\circ} n^2}{2 (T-T_0)}\\
&\leq \frac{2F_Tn^2}{\frac{2}{p_{\min}} + t - t_0} +  \frac{F^{\circ} n^2}{2 (T-T_0)}.
\end{aligned}$$
A sufficient condition to bound the duality gap by $\varepsilon$ is that $T_0 \geq t_0 - \frac{2}{p_{\min}} + \frac{4F^{\circ}n^2}{\varepsilon}$ and $T \geq T_0 + \frac{F^{\circ} n^2}{\varepsilon}$ which also implies $\E[\OA(\alphav^{(T_0)}) - \OA(\alphav^{*})] \leq \varepsilon/2$. Since we also need $T_0\geq t_0$ and $T-T_0 \geq \frac{1}{p_{\min}}$, the overall number of iterations should satisfy:
$$T_0 \geq \max \Big\{t_0, \frac{4F_Tn^2}{\varepsilon} - \frac{2}{p_{\min}} + t_0 \Big\}$$ and 
$$T- T_0 \geq \max \Big\{\frac{1}{p_{\min}}, \frac{F^{\circ}n^2}{\varepsilon}\Big \}.$$
Using $a+b \geq \max(a,b)$ we finally can bound the total number of required iterations to reach a duality gap of $\varepsilon$ by:
$$ \begin{aligned} T &\geq T_0 + \frac{1}{p_{\min}} + \frac{F^{\circ}n^2}{\varepsilon} \\
&\geq t_0 +  \frac{4F^{\circ}n^2}{\varepsilon} - \frac{1}{p_{\min}} + \frac{F^{\circ}n^2}{\varepsilon} \\
&=  t_0 +  \frac{5F^{\circ}n^2}{\varepsilon} - \frac{1}{p_{\min}}\end{aligned}.$$
This concludes the proof.
\end{proof}

\newpage
\begin{reptheorem}{theorem_gapwise}
Let $f$ be a $\frac{1}{\beta}$-smooth function. Then, if $g^*_i$ is $L_i$-Lipschitz for each $i$ and $p_i^{(t)} := \frac{G_i(\alphav^{(t)})}{G(\alphav^{(t)})}$, then the iterations of the CD method satisfies
\begin{equation}
\E[\varepsilon_A^{(t)}] \leq \frac{2F^{\circ}_g n^2+2n\varepsilon_A^{(0)}}{t+2n},
\end{equation}
where $F^{\circ}_g$ is an upper bound on $\E\Big[ F^{(t)}_g\Big],$ where the expectation is taken over the random choice of the sampled coordinate at iterations $1,\dots, t$ of the algorithm. Here $\overrightarrow{\Gv}$ and $\overrightarrow{\Fv}$ are defined as:
$$\overrightarrow{\Gv}:= (G_i(\alphav^{(t)}))_{i=1}^n, \hspace{0.5cm} \overrightarrow{\Fv}:= ( \|\av_i \|^2 |\kappa_i^{(t)}|^2)_{i=1}^n,$$
and $F^{(t)}_g$ is defined analogously to \eqref{F_genconv}:
\begin{equation}
\label{F_gapAgain}
 F^{(t)}_g:=\frac{\chi(\overrightarrow{\Fv})}{n \beta (\chi(\overrightarrow{\Gv}))^3} \sum_i \|\av_i\|^2 |\kappa_i^{(t)}|^2. 
\end{equation}
\end{reptheorem}
\begin{proof}
We start from the result \eqref{eq1Again} of Lemma \ref{lemma1} when $\mu_i=0$:
$$ \begin{aligned}
\E[\OA(\alphav^{(t+1)}) | \alphav^{(t)}] &\leq \OA(\alphav^{(t)}) - \sum_{i} s_i p_i^{(t)} G_i(\alphav^{(t)}) \\&+ \sum_{i} p_i^{(t)} \frac{s_i^2 \|\av_i\|^2}{2 \beta} |\kappa_i^{(t)}|^2, 
\end{aligned}$$
by regrouping the elements and subtracting the optimal function value $\OA(\alphav^\star)$ from both sides we obtain:
$$ \begin{aligned}
&\E[\OA(\alphav^{(t+1)}) - \OA(\alphav^\star)| \alphav^{(t)}] \leq \OA(\alphav^{(t)}) - \OA(\alphav^\star) \\&- \sum_{i} s_i p_i^{(t)} G_i(\alphav^{(t)}) + \sum_{i}  \frac{p_i^t s_i^2 }{2 \beta} \|\av_i\|^2 |\kappa_i^{(t)}|^2.
\end{aligned}$$
With $\varepsilon_A^{(t)} := \OA(\alphav^{t}) - \OA(\alphav^{*})$:
$$ \begin{aligned}
\E[\varepsilon_A^{(t+1)}| \alphav^{(t)}] &\leq \varepsilon_A^{(t)} - \sum_i s_i p_i^{(t)} G_i(\alphav^{(t)}) \\&+ \sum_i \frac{p_i^{(t)}s_i^2}{2 \beta} \|\av_i \|^2 |\kappa_i^{(t)}|^2.
\end{aligned}$$
We take $p_i^{(t)} := \frac{G_i(\alphav^{(t)})}{G(\alphav^{(t)})}$ and $s_i := s$, then:
$$\begin{aligned}
\E[\varepsilon_A^{(t+1)}| \alphav^{(t)}] &\leq \varepsilon_A^{(t)} - \frac{s}{G(\alphav^{(t)})} \sum_i (G_i(\alphav^{(t)}))^2 \\&+ \frac{s^2}{2 \beta G(\alphav^{(t)})}\sum_i  G_i(\alphav^{(t)}) \|\av_i \|^2 |\kappa_i^{(t)}|^2.
\end{aligned}$$
To simplify the following derivation we define a duality gap vector $\overrightarrow{\Gv}:= (G_i(\alphav^{(t)}))_{i=1}^n$ and residual vector $\overrightarrow{\Fv}:= ( \|\av_i \|^2 |\kappa_i^{(t)}|^2)_{i=1}^n$, the inequality becomes:
$$\E[\varepsilon_A^{(t+1)}| \alphav^{(t)}] \leq \varepsilon_A^{(t)} - \frac{s}{G(\alphav^{(t)})} \| \overrightarrow{\Gv}\|_2^2 + \frac{s^2}{2 \beta G(\alphav^{(t)})} \langle \overrightarrow{\Gv},\overrightarrow{\Fv} \rangle.$$
By bounding the last term using the Cauchy-Schwarz inequality $\langle \overrightarrow{\Gv},\overrightarrow{\Fv} \rangle \leq \|\overrightarrow{\Gv}\|_2 \|\overrightarrow{\Fv}\|_2$ and using Lemma \ref{lemmaNormRelation} we obtain:
$$
\begin{aligned} 
&\E[\varepsilon_A^{(t+1)}| \alphav^{(t)}] \leq \varepsilon_A^{(t)} - \frac{s}{G(\alphav^{(t)})} \| \overrightarrow{\Gv}\|_2^2 \\&+ \frac{s^2}{2 \beta G(\alphav^{(t)})} \|\overrightarrow{\Gv}\|_2 \|\overrightarrow{\Fv}\|_2 \\
&= \varepsilon_A^{(t)} - \frac{sG(\alphav^{(t)})(\chi(\overrightarrow{\Gv}))^2}{n}  \\&+ \frac{s^2 \chi(\overrightarrow{\Gv}) \chi(\overrightarrow{\Fv}) \sum_i  \|\av_i \|^2 |\kappa_i^{(t)}|^2}{2 n \beta} \\
&\leq \varepsilon_A^{(t)} - \varepsilon_A^{(t)} \frac{(\chi(\overrightarrow{\Gv}))^2s}{n}  + \frac{s^2 \chi(\overrightarrow{\Gv}) \chi(\overrightarrow{\Fv}) \sum_i  \|\av_i \|^2 |\kappa_i^{(t)}|^2}{2 n \beta} \\
&= \Big( 1 - \frac{(\chi(\overrightarrow{\Gv}))^2s}{n} \Big)\varepsilon_A^{(t)}   + \frac{s^2 \chi(\overrightarrow{\Gv}) \chi(\overrightarrow{\Fv}) \sum_i  \|\av_i \|^2 |\kappa_i^{(t)}|^2}{2 n \beta}. 
\end{aligned}
$$
In the third line we have used weak duality, that is $G(\alphav^{(t)}) \geq \varepsilon_A^{(t)}$.
Analogously to the proof of Theorem \ref{thm:main} we now prove that the suboptimality is bounded by:
\begin{equation}
\label{gap_subopt}
\E[\varepsilon_A^{(t)}] \leq \frac{2F^{\circ}_g n^2+2n\varepsilon_A^{(0)}}{t+2n},
\end{equation}
where
$$ F^{\circ}_g \geq\E\bigg[\frac{\chi(\overrightarrow{\Fv})}{n \beta (\chi(\overrightarrow{\Gv}))^3} \sum_i \|\av_i\|^2 |\kappa_i^{(t)}|^2\bigg].$$
The basis of induction at $t=0$ obviously follows from the non-negativity of $F^{\circ}_g$.\\
Now let us prove the induction step, assume that condition \eqref{gap_subopt} holds at step $t$, then by taking $s:=\frac{2n}{(t+2n)(\chi(\overrightarrow{\Gv}))^2}$ we get:
\begin{equation}
\label{p1}
\begin{aligned}
&\E[\varepsilon_A^{(t+1)} | \alphav^{(t)}] \leq \Big( 1 - \frac{(\chi(\overrightarrow{\Gv}))^2s}{n} \Big) \varepsilon_A^{(t)} \\&+ \frac{s^2 \chi(\overrightarrow{\Gv}) \chi(\overrightarrow{\Fv}) \sum_i  \|\av_i \|^2 |\kappa_i^{(t)}|^2}{2 n \beta}\\
&\leq \Big( 1 - \frac{2}{(t+2n)} \Big) \frac{2F^{\circ}_g n^2+2n\varepsilon_A^{(0)}}{t+2n}  \\& + \frac{2n \chi(\overrightarrow{\Fv}) \sum_i  \|\av_i \|^2 |\kappa_i^{(t)}|^2}{\beta (t+2n)^2  (\chi(\overrightarrow{\Gv}))^3}.
\end{aligned}
\end{equation}
By taking an unconditional expectation of \eqref{p1} and bounding by $\hat{C}:=F^{\circ}_gn+\varepsilon_A^{(0)}$ we obtain:
$$
\begin{aligned}
\E[\varepsilon_A^{(t+1)}] &\leq \Big( 1 - \frac{2}{(t+2n)} \Big) \frac{2n\hat{C}}{t+2n}  \\& + \frac{2n}{ (t+2n)^2} \E\bigg[\frac{ \chi(\overrightarrow{\Fv}) \sum_i  \|\av_i \|^2 |\kappa_i^{(t)}|^2}{\beta  (\chi(\overrightarrow{\Gv}))^3} \bigg]\\
&\leq \Big( 1 - \frac{2}{(t+2n)} \Big) \frac{2n\hat{C}}{t+2n}   + \frac{2n\hat{C}}{ (t+2n)^2}\\
&= \frac{2n\hat{C}}{t+2n} \Big( 1 - \frac{2}{(t+2n)}  + \frac{1}{(t+2n)}\Big)\\
&= \frac{2n\hat{C}}{t+2n} \frac{t+2n-1}{t+2n} \\
&\leq \frac{2n\hat{C}}{t+2n} \frac{t+2n}{t+2n+1} \\
&= \frac{2n\hat{C}}{t+2n+1}.
\end{aligned}
$$
And this concludes the proof.
\end{proof}

\newpage
\section{Algorithms}
\label{app:algorithms}
\subsection{Algorithms with fixed sampling}
In this subsection we give the coordinate descent algorithms with sampling schemes with fixed probabilities which we derived the in previous chapter. The basic Coordinate Descent with uniform sampling for Lasso was presented in \cite[Algorithm~1]{ShalevShwartz:2011vo}. Coordinate descent for hinge-loss SVM was given in \cite{ShalevShwartz:2013wl}. Here we give these algorithms along with their enhanced fixed non-uniform sampling versions of \textit{importance sampling} and heuristic \textit{gap-per-epoch} sampling, which is based on initial coordinate-wise duality gaps at the beginning of each epoch.\\
\paragraph{Lasso}
See Algorithms \ref{lassoCD}, \ref{lassoCD_gap-per-epoch}. To describe the algorithm the "soft-threshold" function $s_\tau(w)$ is defined:
$$s_\tau(w) := sign(w)(|w|-\tau)_+ = sign(w) \max \Big \{ |w|-\tau,0 \Big\}$$
\begin{algorithm}[H]
    \caption{Coordinate Descent for Lasso (uniform \& importance)}
    \label{lassoCD}
    \begin{algorithmic}[1]
    \State Choose $\text{mode} \in [\text{uniform}, \text{importance}]$
       \State let $\alphav^{(0)}=0,$ $\wv^{(0)} = \nabla f (A \alphav^{(0)})$
   \Switch{$\text{mode}$}
    \Case{$\text{uniform}$}
       \State $p_i := \frac{1}{n} ~~~\forall i$ 
    \EndCase
    \Case{$\text{importance}$}
      \State $p_i :=  \frac{L_i \|\av_i\|}{\sum_j L_j \|\av_j\|} ~~~\forall i$ 
    \EndCase
  \EndSwitch
		\For{ t = 0,1,...}
		\State sample $j$ from $[n]$ according to distribution $\pv$
		\State let $z_j = \frac {\partial f(A \alphav^{(t)})}{\partial \alpha_j}$ %
		\State  $\alpha_j^{(t+1)} = s_{\lambda}(\alpha_j^{(t)} - z_j)$
		\State $\wv^{(t+1)} = \nabla f (A \alphav^{(t+1)})$
		\EndFor
    \end{algorithmic}
\end{algorithm}

\begin{algorithm}[H]
    \caption{Coordinate Descent for Lasso (gap-per-epoch)}
    \label{lassoCD_gap-per-epoch}
    \begin{algorithmic}[1]
       \State let $\alphav^{(0)}=0,$ $\wv^{(0)} = \nabla f (A \alphav^{(0)})$
		\For{ t = 0,1,...}
		 \If{$\text{mod}(t,n) == 0$} 
    	\State generate probabilities distribution $\pv^{(t)}$:
    	$$p_i = \frac{B\big [|\av_i^\top\wv^{(0)}| - \lambda \big]_+  + \lambda|\alpha_i^{(0)}| + \alpha_i^{(0)} \av_i^\top\wv^{(0)}}{\sum_j \Big ( B\big [|\av_j^\top\wv^{(0)}| - \lambda \big]_+  +\lambda |\alpha_j^{(0)}| + \alpha_j^{(0)} \av_j^\top\wv^{(0)} \Big )}$$ 
    	\EndIf
		\State sample $j$ from $[n]$ according to distribution $\pv$
		\State let $z_j = \frac {\partial f(A \alphav^{(t)})}{\partial \alpha_j}$ 
		\State  $\alpha_j^{(t+1)} = s_{\lambda}(\alpha_j^{(t)} - z_j)$
		\State $\wv^{(t+1)} = \nabla f (A \alphav^{(t+1)})$
		\EndFor
    \end{algorithmic}
\end{algorithm}

\paragraph{Hinge-Loss SVM}
See Algorithms \ref{svmSDCD}, \ref{svmSDCD_gap-per-epoch}
\begin{algorithm}[H] %
    \caption{Stochastic Dual Coordinate Descent (uniform \& importance)}
    \label{svmSDCD}
    \begin{algorithmic}[1]
    	\State Choose $\text{mode} \in [\text{uniform}, \text{importance}]$
       \State let $\alphav^{(0)}=0,$ $\wv^{(0)} = 0$
       \Switch{$\text{mode}$}
    \Case{$\text{uniform}$}
       \State $p_i := \frac{1}{n} ~~~\forall i$ 
    \EndCase
    \Case{$\text{importance}$}
      \State $p_i :=  \frac{L_i \|\av_i\|}{\sum_j L_j \|\av_j\|} ~~~\forall i$ 
    \EndCase
  \EndSwitch
		\For{ t = 0,1,...}
		\State sample $j$ from $[n]$ according to distribution $\pv$
		\State $\Delta \alpha_j = y_j \max\bigg(0, \min \Big(1, \frac{1 - y_j \av_j^\top \wv^{(t)} }{\|\av_j\|^2 / (\lambda n) } + y_j \alpha_j^{(t)} \Big) \bigg) - \alpha_j^{(t)}$
		\State $\alphav^{(t+1)} = \alphav^{(t)} + \Delta \alpha_j \ev_j$    
		\State $\wv^{(t+1)} =\wv^{(t)} + (\lambda n)^{-1} \Delta \alpha_j \av_j$
		\EndFor
    \end{algorithmic}
\end{algorithm}

\begin{algorithm}[H] %
    \caption{Stochastic Dual Coordinate Descent (gap-per-epoch)}
    \label{svmSDCD_gap-per-epoch}
    \begin{algorithmic}[1]
       \State let $\alphav^{(0)}=0,$ $\wv^{(0)} = 0$
		\For{ t = 0,1,...}
		\If{$\text{mod}(t,n) == 0$} 
    	\State generate probabilities distribution $\pv^{(t)}$:
    	$$p_i = \frac{\phi_i(\av_i^\top\wv) -  \alpha_iy_i + \alpha_i\av_i^\top \wv }{\sum_{j=1}^n \Big ( \phi_j(\av_j^\top\wv) -  \alpha_jy_j + \alpha_j\av_j^\top \wv \Big )}$$ 
    	\EndIf
		\State sample $j$ from $[n]$ according to distribution $\pv$
		\State $\Delta \alpha_j = y_j \max\bigg(0, \min \Big(1, \frac{1 - y_j \av_j^\top \wv^{(t)} }{\|\av_j\|^2 / (\lambda n) } + y_j \alpha_j^{(t)} \Big) \bigg) - \alpha_j^{(t)}$
		\State $\alphav^{(t+1)} = \alphav^{(t)} + \Delta \alpha_j \ev_j$    
		\State $\wv^{(t+1)} =\wv^{(t)} + (\lambda n)^{-1} \Delta \alpha_j \av_j$
		\EndFor
    \end{algorithmic}
\end{algorithm}

\newpage
\subsection{Algorithms with adaptive sampling}
In this subsection we consider Coordinate Descent with adaptive sampling schemes. Here we present 4 different schemes:
\begin{itemize}
\item \textbf{supportSet-uniform} discussed in Section \ref{mix_theory}, defined in \eqref{ssunifdist}.
\item \textbf{adaptive} discussed in Section \ref{mix_theory}, defined in \eqref{adadist}.
\item \textbf{ada-uniform} discussed in Section \ref{mix_theory}, defined in \eqref{mixdist}.
\item \textbf{ada-gap} discussed in Section \ref{gapwise_theory}, defined in Theorem \ref{theorem_gapwise}.
\end{itemize}
The algorithms with aforementioned sampling schemes are given below. For Lasso see Algorithms \ref{lasso_ssuniform}, \ref{lasso_adaptive}, \ref{lasso_ada-unif} and \ref{lasso_gapwise}. For hinge-loss SVM see Algorithms \ref{svm_ssuniform}, \ref{svm_adaptive}, \ref{svm_ada-unif} and \ref{svm_gapwise}.

\begin{algorithm}[H]
\caption{Coordinate Descent (supportSet-uniform)}
 \label{lasso_ssuniform}
\begin{algorithmic}[1]
\State let $\alphav^{(0)}=0,$ $\wv^{(0)} = \nabla f (A \alphav^{(0)})$
\For{ t = 0,1,...}
\State calculate absolute values of dual residuals $|\kappa_j^{(t)}|$ for all $j \in [n]$ %
$$|\kappa_j^{(t)}| := \Big |\alpha_j - B \cdot \text{sign}(\av_i^\top \wv^{(t)}) \cdot \big [|\av_j^\top\wv^{(t)}| - \lambda \big]_+ \Big| $$
\State find $t$-support set $I_t = \{ i \in [n] : \kappa_i^{(t)} \neq 0 \}  \subseteq [n]$
\State generate adaptive probabilities distribution $\pv^{(t)}$, for each $i\in[n]$:
 $$p_i^{(t)} := \begin{cases}
    \frac{1}{|I_t|},& \text{if } \kappa_i^{(t)} \neq 0 \\
    0,& \text{otherwise}
\end{cases}$$
\State sample $j$ from $[n]$ according to $\pv^{(t)}$
\State let $z_j = \frac {\partial f(A \alphav^{(t)})}{\partial \alpha_j}$ %
\State  $\alpha_j^{(t+1)} = s_{\lambda}(\alpha_j^{(t)} - z_j)$
\State $\wv^{(t+1)} := \nabla f (A \alphav^{(t+1)})$
\EndFor
\end{algorithmic}
\end{algorithm}

\begin{algorithm}[H]
\caption{Coordinate Descent (adaptive)}
 \label{lasso_adaptive}
\begin{algorithmic}[1]
\State let $\alphav^{(0)}=0,$ $\wv^{(0)} = \nabla f (A \alphav^{(0)})$
\For{ t = 0,1,...}
\State calculate absolute values of dual residuals $|\kappa_j^{(t)}|$ for all $j \in [n]$
$$|\kappa_j^{(t)}| = \Big |\alpha_j - B \cdot \text{sign}(\av_i^\top \wv^{(t)}) \cdot \big [|\av_j^\top\wv^{(t)}| - \lambda \big]_+ \Big| $$
\State generate adaptive probabilities distribution $\pv^{(t)}$:
 $$p_i^{(t)} = \frac{|\kappa_i^{(t)}| \|\av_i\|}{\sum_j |\kappa_j^{(t)}| \|\av_j\|}$$
\State sample $j$ from $[n]$ according to $\pv^{(t)}$
\State let $z_j = \frac {\partial f(A \alphav^{(t)})}{\partial \alpha_j}$
\State  $\alpha_j^{(t+1)} = s_{\lambda}(\alpha_j^{(t)} - z_j)$
\State $\wv^{(t+1)} = \nabla f (A \alphav^{(t+1)})$
\EndFor
\end{algorithmic}
\end{algorithm}

\begin{algorithm}[H]
\caption{Coordinate Descent (ada-uniform)}
 \label{lasso_ada-unif} %
\begin{algorithmic}[1]
\State let $\alphav^{(0)}=0,$ $\wv^{(0)} = \nabla f (A \alphav^{(0)})$
\For{ t = 0,1,...}
\State calculate absolute values of dual residuals $|\kappa_j^{(t)}|$ for all $j \in [n]$
$$|\kappa_j^{(t)}| = \Big |\alpha_j - B \cdot \text{sign}(\av_i^\top \wv^{(t)}) \cdot \big [|\av_j^\top\wv^{(t)}| - \lambda \big]_+ \Big| $$
\State find $t$-support set $I_t = \{ i \in [n] : \kappa_i^{(t)} \neq 0 \}  \subseteq [n]$
\State generate adaptive probabilities distribution $\pv^{(t)}$:
 $$\begin{cases}
    p_i^{(t)} =  \frac{1}{2|I_t|} + \frac{|\kappa_i^{(t)}| \|\av_i\|}{2\sum_j |\kappa_j^{(t)}| \|\av_j\|},& \text{if } \kappa_i^{(t)} \neq 0 \\
    p_i^{(t)} = 0,& \text{otherwise}
\end{cases}$$
\State sample $j$ from $[n]$ according to $\pv^{(t)}$
\State let $z_j = \frac {\partial f(A \alphav^{(t)})}{\partial \alpha_j}$
\State  $\alpha_j^{(t+1)} = s_{\lambda}(\alpha_j^{(t)} - z_j)$
\State $\wv^{(t+1)} = \nabla f (A \alphav^{(t+1)})$
\EndFor
\end{algorithmic}
\end{algorithm}

\begin{algorithm}[H]
\caption{Coordinate Descent (ada-gap)}
 \label{lasso_gapwise}
\begin{algorithmic}[1]
\State let $\alphav^{(0)}=0,$ $\wv^{(0)} = \nabla f (A \alphav^{(0)})$
\For{ t = 0,1,...}
\State calculate feature-wise duality gaps $G_j^{(t)}$ for all $j \in [n]$
$$G_j^{(t)} = B\big [|\av_i^\top\wv^{(t)}| - \lambda \big]_+  + \lambda |\alpha_i^{(t)}| + \alpha_i^{(t)} \av_i^\top\wv^{(t)} $$
\State generate adaptive probabilities distribution $\pv^{(t)}$:
 $$p_i^{(t)} = \frac{G_i^{(t)}}{\sum_j G_j^{(t)}}$$
\State sample $j$ from $[n]$ according to $\pv^{(t)}$
\State let $z_j = \frac {\partial f(A \alphav^{(t)})}{\partial \alpha_j}$
\State  $\alpha_j^{(t+1)} = s_{\lambda}(\alpha_j^{(t)} - z_j)$
\State $\wv^{(t+1)} = \nabla f (A \alphav^{(t+1)})$
\EndFor
\end{algorithmic}
\end{algorithm}

\begin{algorithm}[H] 
    \caption{Stochastic Dual Coordinate Descent (supportSet-uniform)}
    \label{svm_ssuniform}
    \begin{algorithmic}[1]
       \State let $\alphav^{(0)}=0,$ $\wv^{(0)} = 0$
		\For{ t = 0,1,...} 
		\State calculate absolute values of dual residuals $|\kappa_j^{(t)}|$ for all $j \in [n]$ 
    	\State find $t$-support set $I_t = \{ i \in [n] : \kappa_i^{(t)} \neq 0 \}  \subseteq [n]$
\State generate adaptive probabilities distribution $\pv^{(t)}$, for each $i\in[n]$:
 $$p_i^{(t)} := \begin{cases}
    \frac{1}{|I_t|},& \text{if } \kappa_i^{(t)} \neq 0 \\
    0,& \text{otherwise}
\end{cases}$$
		\State sample $j$ from $[n]$ according to distribution $\pv^{(t)}$
		\State $\Delta \alpha_j = y_j \max\bigg(0, \min \Big(1, \frac{1 - y_j \av_j^\top \wv^{(t)} }{\|\av_j\|^2 / (\lambda n) } + y_j \alpha_j^{(t)} \Big) \bigg) - \alpha_j^{(t)}$
		\State $\alphav^{(t+1)} = \alphav^{(t)} + \Delta \alpha_j \ev_j$    
		\State $\wv^{(t+1)} =\wv^{(t)} + (\lambda n)^{-1} \Delta \alpha_j \av_j$
		\EndFor
    \end{algorithmic}
\end{algorithm}

\begin{algorithm}[H] 
    \caption{Stochastic Dual Coordinate Descent (adaptive)}
    \label{svm_adaptive}
    \begin{algorithmic}[1]
       \State let $\alphav^{(0)}=0,$ $\wv^{(0)} = 0$
		\For{ t = 0,1,...} 
		\State calculate absolute values of dual residuals $|\kappa_j^{(t)}|$ for all $j \in [n]$ 
    	\State generate adaptive probabilities distribution $\pv^{(t)}$:
    	 $$p_i^{(t)} = \frac{|\kappa_i^{(t)}| \|\av_i\|}{\sum_j |\kappa_j^{(t)}| \|\av_j\|}$$
		\State sample $j$ from $[n]$ according to distribution $\pv^{(t)}$
		\State $\Delta \alpha_j = y_j \max\bigg(0, \min \Big(1, \frac{1 - y_j \av_j^\top \wv^{(t)} }{\|\av_j\|^2 / (\lambda n) } + y_j \alpha_j^{(t)} \Big) \bigg) - \alpha_j^{(t)}$
		\State $\alphav^{(t+1)} = \alphav^{(t)} + \Delta \alpha_j \ev_j$    
		\State $\wv^{(t+1)} =\wv^{(t)} + (\lambda n)^{-1} \Delta \alpha_j \av_j$
		\EndFor
    \end{algorithmic}
\end{algorithm}

\begin{algorithm}[H] 
    \caption{Stochastic Dual Coordinate Descent (ada-uniform)}
    \label{svm_ada-unif}
    \begin{algorithmic}[1]
       \State let $\alphav^{(0)}=0,$ $\wv^{(0)} = 0$
		\For{ t = 0,1,...} 
		\State calculate absolute values of dual residuals $|\kappa_j^{(t)}|$ for all $j \in [n]$ 
    	\State find $t$-support set $I_t = \{ i \in [n] : \kappa_i^{(t)} \neq 0 \}  \subseteq [n]$
\State generate adaptive probabilities distribution $\pv^{(t)}$:
 $$\begin{cases}
    p_i^{(t)} =  \frac{1}{2|I_t|} + \frac{|\kappa_i^{(t)}| \|\av_i\|}{2\sum_j |\kappa_j^{(t)}| \|\av_j\|},& \text{if } \kappa_i^{(t)} \neq 0 \\
    p_i^{(t)} = 0,& \text{otherwise}
\end{cases}$$
		\State sample $j$ from $[n]$ according to distribution $\pv^{(t)}$
		\State $\Delta \alpha_j = y_j \max\bigg(0, \min \Big(1, \frac{1 - y_j \av_j^\top \wv^{(t)} }{\|\av_j\|^2 / (\lambda n) } + y_j \alpha_j^{(t)} \Big) \bigg) - \alpha_j^{(t)}$
		\State $\alphav^{(t+1)} = \alphav^{(t)} + \Delta \alpha_j \ev_j$    
		\State $\wv^{(t+1)} =\wv^{(t)} + (\lambda n)^{-1} \Delta \alpha_j \av_j$
		\EndFor
    \end{algorithmic}
\end{algorithm}

\begin{algorithm}[H] 
    \caption{Stochastic Dual Coordinate Descent (ada-gap)}
    \label{svm_gapwise}
    \begin{algorithmic}[1]
       \State let $\alphav^{(0)}=0,$ $\wv^{(0)} = 0$
		\For{ t = 0,1,...} 
    	\State generate adaptive probabilities distribution $\pv^{(t)}$:
    	$$p_i = \frac{\phi_i(\av_i^\top\wv) -  \alpha_iy_i + \alpha_i\av_i^\top \wv }{\sum_{j=1}^n \Big ( \phi_j(\av_j^\top\wv) -  \alpha_jy_j + \alpha_j\av_j^\top \wv \Big )}$$ 
		\State sample $j$ from $[n]$ according to distribution $\pv^{(t)}$
		\State $\Delta \alpha_j = y_j \max\bigg(0, \min \Big(1, \frac{1 - y_j \av_j^\top \wv^{(t)} }{\|\av_j\|^2 / (\lambda n) } + y_j \alpha_j^{(t)} \Big) \bigg) - \alpha_j^{(t)}$
		\State $\alphav^{(t+1)} = \alphav^{(t)} + \Delta \alpha_j \ev_j$    
		\State $\wv^{(t+1)} =\wv^{(t)} + (\lambda n)^{-1} \Delta \alpha_j \av_j$
		\EndFor
    \end{algorithmic}
\end{algorithm}

\newpage
\section{Experimental results on large dataset}
\label{app:expBIG}
A comparison of the algorithms with fixed per-epoch sampling on large dataset \textit{rcv1} (see Table \ref{dataBIG_lasso}) is given in Figures \ref{fig:results_rcv1BIG} and \ref{fig:results_rcv1BIG_SVM}. We use $\lambda = 7\cdot 10^{-4}$ for Lasso and $\lambda = 0.1$ for SVM.
\begin{table}[H]
\caption{Datasets}\vspace{-2mm}
\label{dataBIG_lasso}
{\small
\begin{tabular}{|l| l| l| l| l|}%
\hline
Dataset & $d$ & $n$ & nnz$/(nd)$ & $c_v = \frac{\mu(\|\av_i\|)}{\sigma(\|\av_i\|)}$ \tabularnewline %
\hline
rcv1 & 47236 & 20242 & 0.16\% & 0.57 \tabularnewline 
\hline 
\end{tabular}
}%
\end{table}
\begin{figure}[h]
\centering
\begin{subfigure}{.25\textwidth}
  \centering
\includegraphics[width=0.9\linewidth]{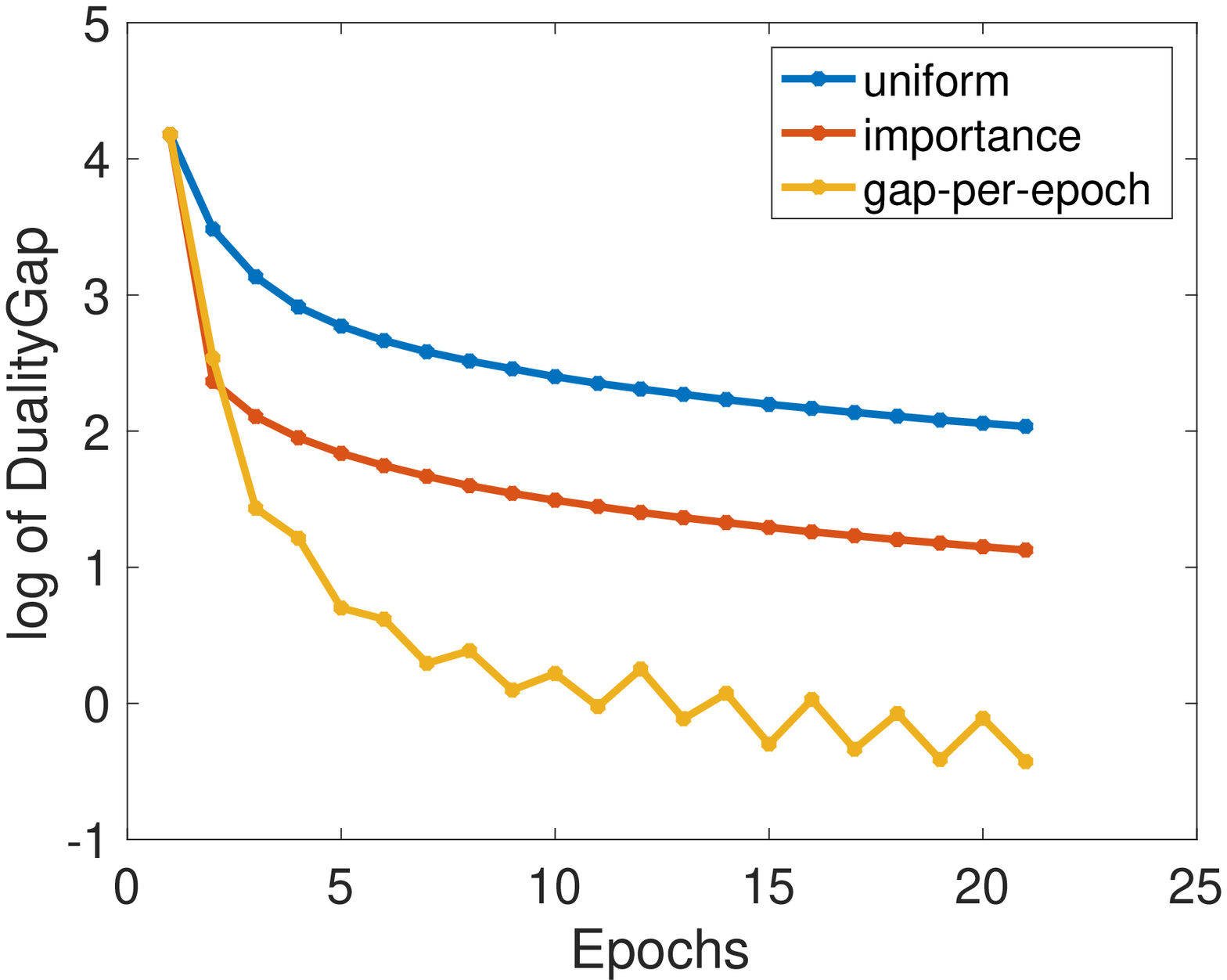}
  \caption{Gap}
\end{subfigure}%
\begin{subfigure}{.25\textwidth}
  \centering
\includegraphics[width=0.9\linewidth]{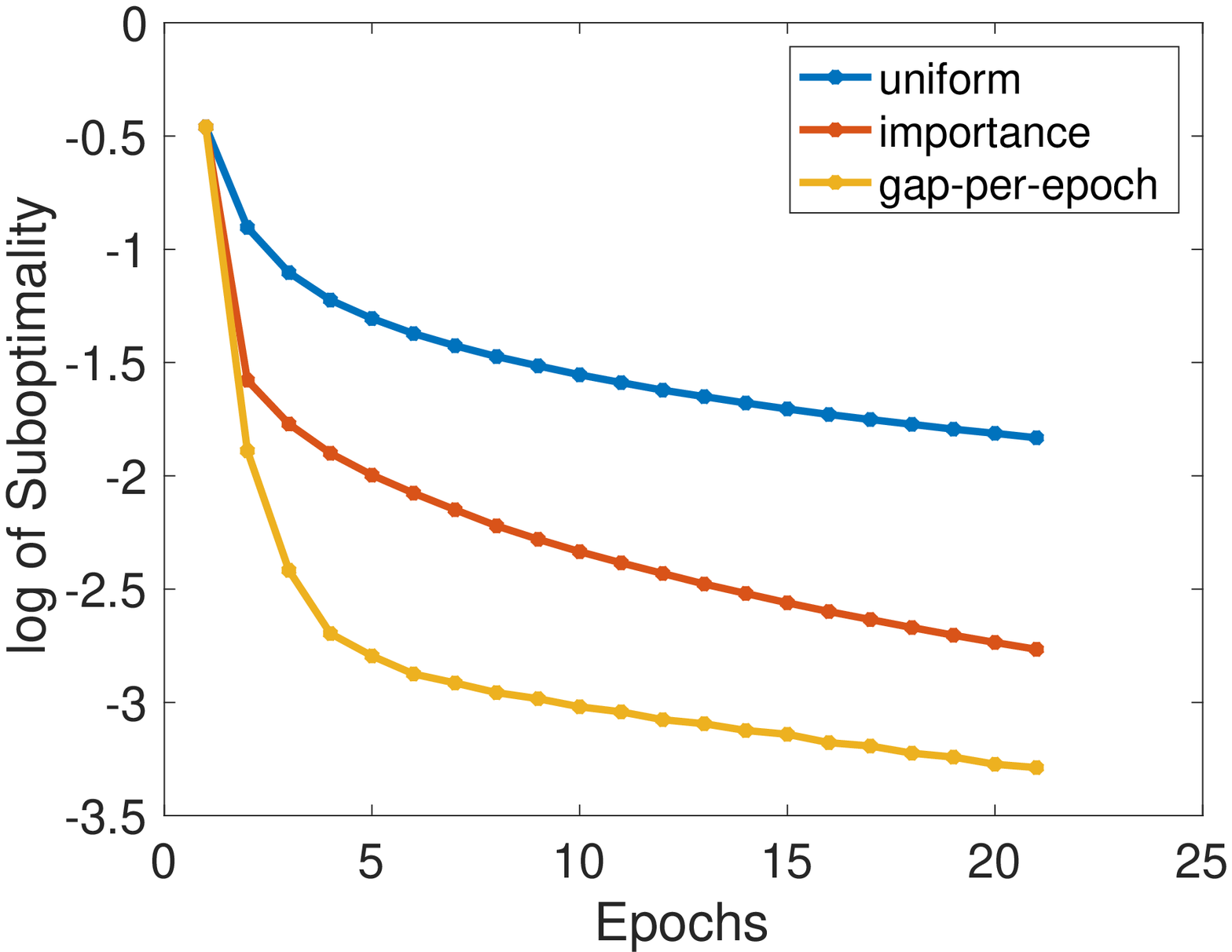}
  \caption{Suboptimality}
\end{subfigure}%
\caption{Lasso on the full \textit{rcv1} dataset. Performance in terms of duality gap and suboptimality}
\label{fig:results_rcv1BIG}
\end{figure}

\begin{figure}[h]
\centering
\begin{subfigure}{.25\textwidth}
  \centering
\includegraphics[width=0.9\linewidth]{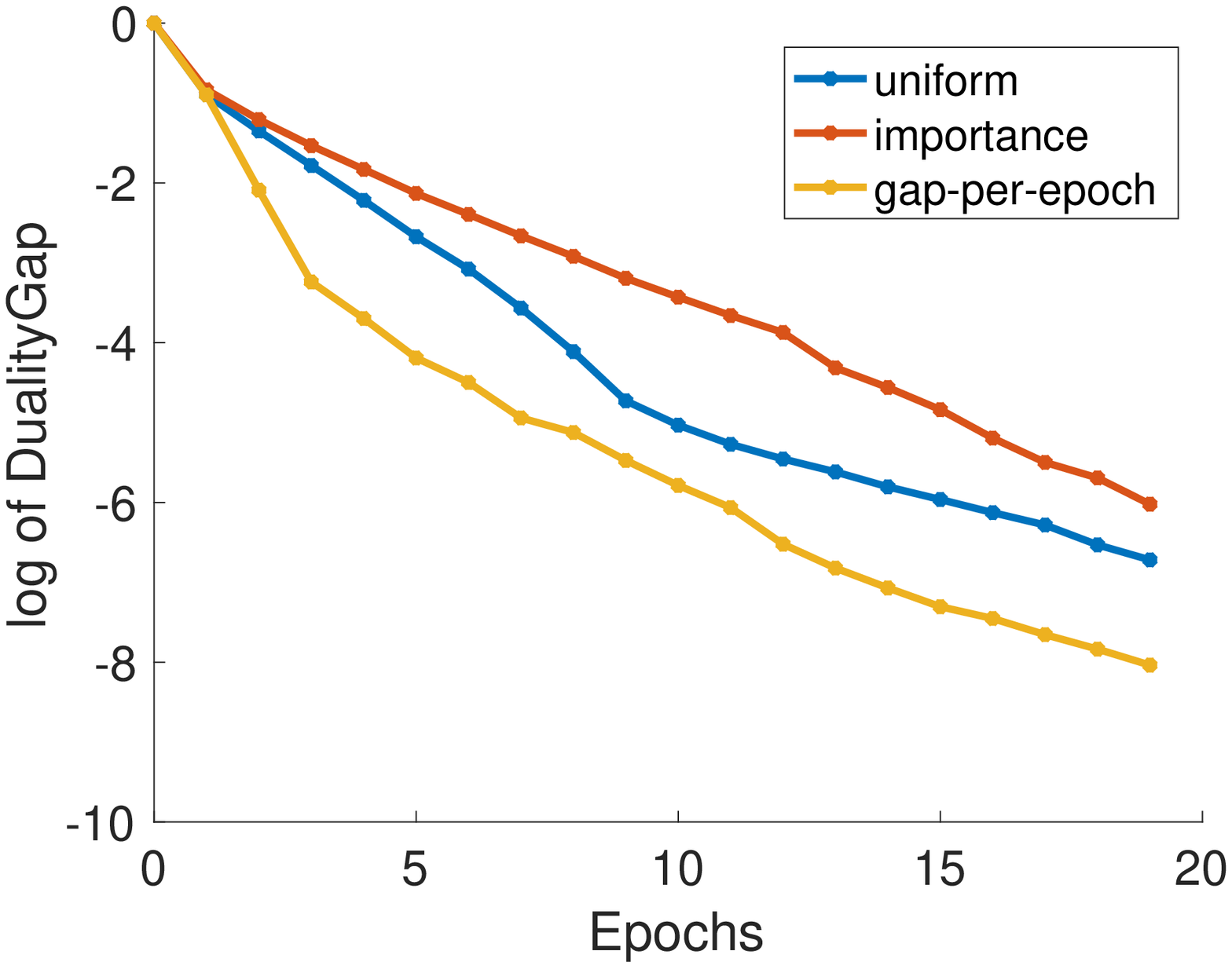}
  \caption{Gap}
\end{subfigure}%
\begin{subfigure}{.25\textwidth}
  \centering
\includegraphics[width=0.9\linewidth]{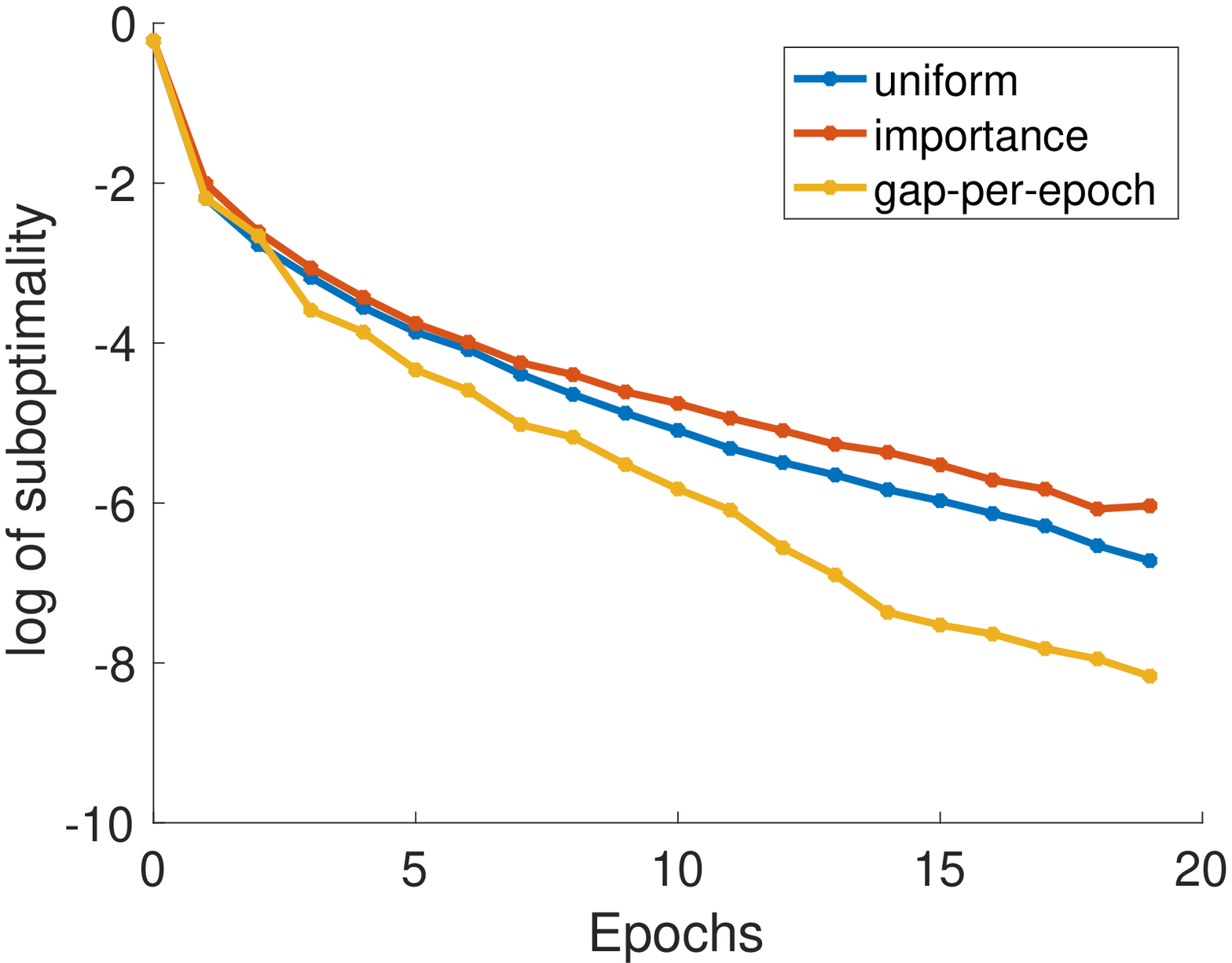}
  \caption{Suboptimality}
\end{subfigure}%
\caption{SVM on the full \textit{rcv1} dataset. Performance in terms of duality gap and suboptimality}
\label{fig:results_rcv1BIG_SVM}
\end{figure}
\end{document}